\documentclass{article}

\PassOptionsToPackage{numbers, compress}{natbib}

\usepackage[ruled,linesnumbered]{algorithm2e}

\usepackage[final]{neurips_2019}

\usepackage[utf8]{inputenc} 
\usepackage[T1]{fontenc}    
\usepackage{hyperref}       
\usepackage{url}            
\usepackage{booktabs}       
\usepackage{amsfonts}       
\usepackage{nicefrac}       
\usepackage{microtype}      

\usepackage{graphicx}
\usepackage{wrapfig}
\usepackage{subcaption}

\usepackage[dvipsnames]{xcolor}
\usepackage{tablefootnote}
\usepackage{amsmath}
\usepackage{bbold} 
\usepackage{bbm}
\usepackage{amsthm}
\usepackage{txfonts}
\usepackage{mathtools}
\usepackage{graphbox}

\newtheorem{theorem}{Theorem}[section]
\newtheorem{lemma}[theorem]{Lemma}
\newtheorem{definition}[theorem]{Definition}

\DeclareMathOperator*{\argmax}{arg\,max}

\newcommand{\Dtrain}{\mathcal{D}_{\textbf{train}}}
\newcommand{\Dpool}{\mathcal{D}_{\textbf{pool}}}
\newcommand{\Model}{\mathcal{M}}

\newcommand{\theAF}{a}
\newcommand{\BALDAF}{\theAF_{\mathtt{BALD}}}
\newcommand{\BatchBALDAF}{\theAF_{\mathtt{BatchBALD}}}

\newcommand{\y}{y}
\newcommand{\x}{\boldsymbol{x}}

\newcommand{\yi}{y_i}

\newcommand{\sxbs}{\x_{1}^{*}, ..., \x_{\numB}^{*}}

\newcommand{\ybv}{y_{1:b}}
\newcommand{\ybs}{y_1, ..., y_b}
\newcommand{\ybi}{y_{b+1}}

\newcommand{\ybiv}{y_{1:b+1}}

\newcommand{\xbi}{\x_{b+1}}
\newcommand{\xbv}{\x_{1:b}}
\newcommand{\xbiv}{\x_{1:b+1}}
\newcommand{\xbs}{\x_{1}, ..., \x_{b}}

\newcommand{\yn}{y_n}
\newcommand{\ynv}{y_{1:n}}
\newcommand{\yns}{y_1, ..., y_n}

\newcommand{\ynds}{y_1, ..., y_{n-1}}
\newcommand{\yndv}{y_{1:n-1}}

\newcommand{\xn}{\x_{n}}
\newcommand{\xnd}{\x_{n-1}}

\newcommand{\xns}{\x_1, ..., \x_n}

\newcommand{\xndv}{\x_{1:n-1}}
\newcommand{\xnds}{\x_1, ..., \x_{n-1}}

\newcommand{\ii}{i}
\newcommand{\jj}{j}

\newcommand{\numK}{k}
\newcommand{\numC}{c}
\newcommand{\numB}{b}
\newcommand{\numM}{m}
\newcommand{\numN}{n}

\DeclareMathOperator{\opE}{\mathbbm{E}}
\DeclareMathOperator{\opH}{\mathbbm{H}}
\DeclareMathOperator{\opI}{\mathbbm{I}}
\DeclareMathOperator{\opp}{p}
\DeclareMathOperator{\opq}{q}
\DeclareMathOperator{\opmus}{\mu^*}

\newcommand{\im}[1]{\opmus \mathopen{}\left ( #1 \right )\mathclose{}}
\newcommand{\Entropy}[1]{\opH ( #1 )}
\newcommand{\Hc}[2]{\opH ( #1 \mathbin{\vert} #2 )}
\newcommand{\MI}[2]{\opI ( #1 \, ; \, #2 )}
\newcommand{\MIc}[3]{\opI ( #1 \, ; \, #2 \mathbin{\vert} #3  )}
\newcommand{\prob}[1]{\opp ( #1 )}
\newcommand{\probc}[2]{\opp ( #1 \mathbin{\vert} #2 )}

\newcommand{\qprobc}[2]{\opq ( #1 \mathbin{\vert} #2 )}

\newcommand{\E}[2]{\opE_{#2} \left [ #1 \right ]}
\newcommand{\chainedE}[2]{\opE_{#2} {#1}}

\newcommand{\w}{\pmb{\omega}}
\newcommand{\pw}{\prob{\w}}

\newcommand{\wj}{\hat{\pmb{\omega}}_\jj}

\title{BatchBALD: Efficient and Diverse Batch Acquisition for Deep Bayesian Active Learning}

\author{%
  Andreas Kirsch\thanks{joint first authors} \hspace{5mm}
  Joost van Amersfoort\footnotemark[1] \hspace{5mm} Yarin Gal \\
  OATML \\
  Department of Computer Science \\
  University of Oxford \\
  \texttt{\{andreas.kirsch, joost.van.amersfoort, yarin\}@cs.ox.ac.uk}
}

\begin{document}

\maketitle

\begin{abstract}
We develop BatchBALD, a tractable approximation to the mutual information
between a batch of points and model parameters, which we use as an acquisition
function to select multiple informative points jointly for the task of deep
Bayesian active learning. BatchBALD is a greedy linear-time $1 -
\nicefrac{1}{e}$-approximate algorithm amenable to dynamic programming and
efficient caching. We compare BatchBALD to the commonly used approach for batch
data acquisition and find that the current approach acquires similar and
redundant points, sometimes performing worse than randomly acquiring data. We
finish by showing that, using BatchBALD to consider dependencies within an
acquisition batch, we achieve new state of the art performance on standard
benchmarks, providing substantial data efficiency improvements in batch
acquisition.
\end{abstract}

\section{Introduction}
A key problem in deep learning is data efficiency. While excellent performance
can be obtained with modern tools, these are often data-hungry, rendering the
deployment of deep learning in the real-world challenging for many tasks.
Active learning (AL) \citep{cohn1996active} is a powerful technique for
attaining data efficiency. Instead of a-priori collecting and labelling a large
dataset, which often comes at a significant expense, in AL we iteratively
acquire labels from an expert only for the most informative data points from a
pool of available unlabelled data. After each acquisition step, the newly
labelled points are added to the training set, and the model is retrained. This
process is repeated until a suitable level of accuracy is achieved. The goal of AL
is to minimise the amount of data that needs to be labelled.
AL has already made real-world impact in manufacturing
\citep{tong2001active}, robotics \citep{calinon2007learning}, recommender
systems  \citep{adomavicius2005toward}, medical imaging \citep{hoi2006batch},
and NLP \citep{siddhant2018deep}, motivating the need for pushing AL even
further.

In AL, the informativeness of new points is assessed by an \emph{acquisition
function}. There are a number of intuitive choices, such as model uncertainty
and mutual information, and, in this paper, we focus on BALD
\citep{houlsby2011bayesian}, which has proven itself in the context of deep
learning \citep{gal2017deep, shen2017deep, janz2017actively}. BALD is based on
mutual information and scores points based on how well their label would inform
us about the true model parameter distribution. In deep learning models
\citep{he2016deep, simonyan2014very}, we generally treat the parameters as point
estimates instead of distributions. However, Bayesian neural networks have
become a powerful alternative to traditional neural networks and do provide a
distribution over their parameters. Improvements in approximate inference
\citep{blundell2015weight,gal2016dropout} have enabled their usage for high
dimensional data such as images and in conjunction with BALD for Bayesian AL of
images \citep{gal2017deep}.

In practical AL applications, instead of single data points, batches of data
points are acquired during each acquisition step to reduce the number of times
the model is retrained and expert-time is requested. Model retraining becomes a
computational bottleneck for larger models while expert time is
expensive: consider, for example, the effort that goes into commissioning a medical specialist to label
 a single MRI scan, then waiting until the model is retrained, and then commissioning a new
medical specialist to label the next MRI scan, and the extra amount of time this takes.

In \citet{gal2017deep}, \emph{batch acquisition}, i.e. the acquisition of
multiple points, takes the top $\numB$ points with the highest BALD acquisition
score. This naive approach leads to acquiring points that are individually very
informative, but not necessarily so jointly. See figure
\ref{acquisition_example} for such a batch acquisition of BALD in which it performs
poorly whereas scoring points jointly ("BatchBALD") can find \emph{batches} of
informative data points. Figure \ref{rmnist_graph} shows how a dataset
consisting of repeated MNIST digits (with added Gaussian noise) leads BALD to
perform worse than random acquisition while BatchBALD sustains good performance.
\begin{figure}[!tbp]
	\begin{minipage}[t]{0.49\textwidth}
		\centering
		\includegraphics[width=0.8\linewidth]{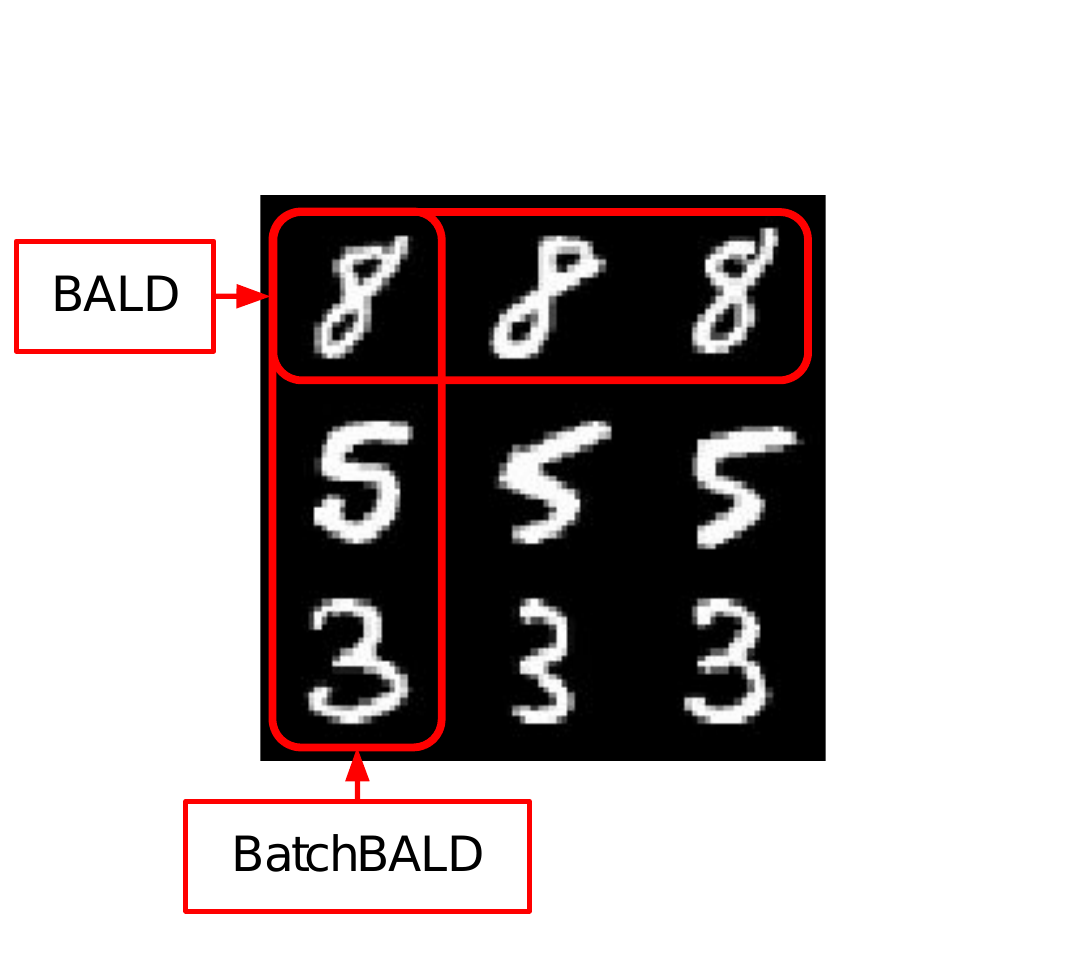}
		\caption{\emph{Idealised acquisitions of BALD and BatchBALD.}
		If a dataset
		were to contain many (near) replicas for each data point, then BALD would select
		all replicas of a single informative data point at the expense of other
		informative data points, wasting data efficiency. }
		\label{acquisition_example}
	\end{minipage}
	\hfill
	\begin{minipage}[t]{0.49\textwidth}
		\includegraphics[width=0.95\linewidth]{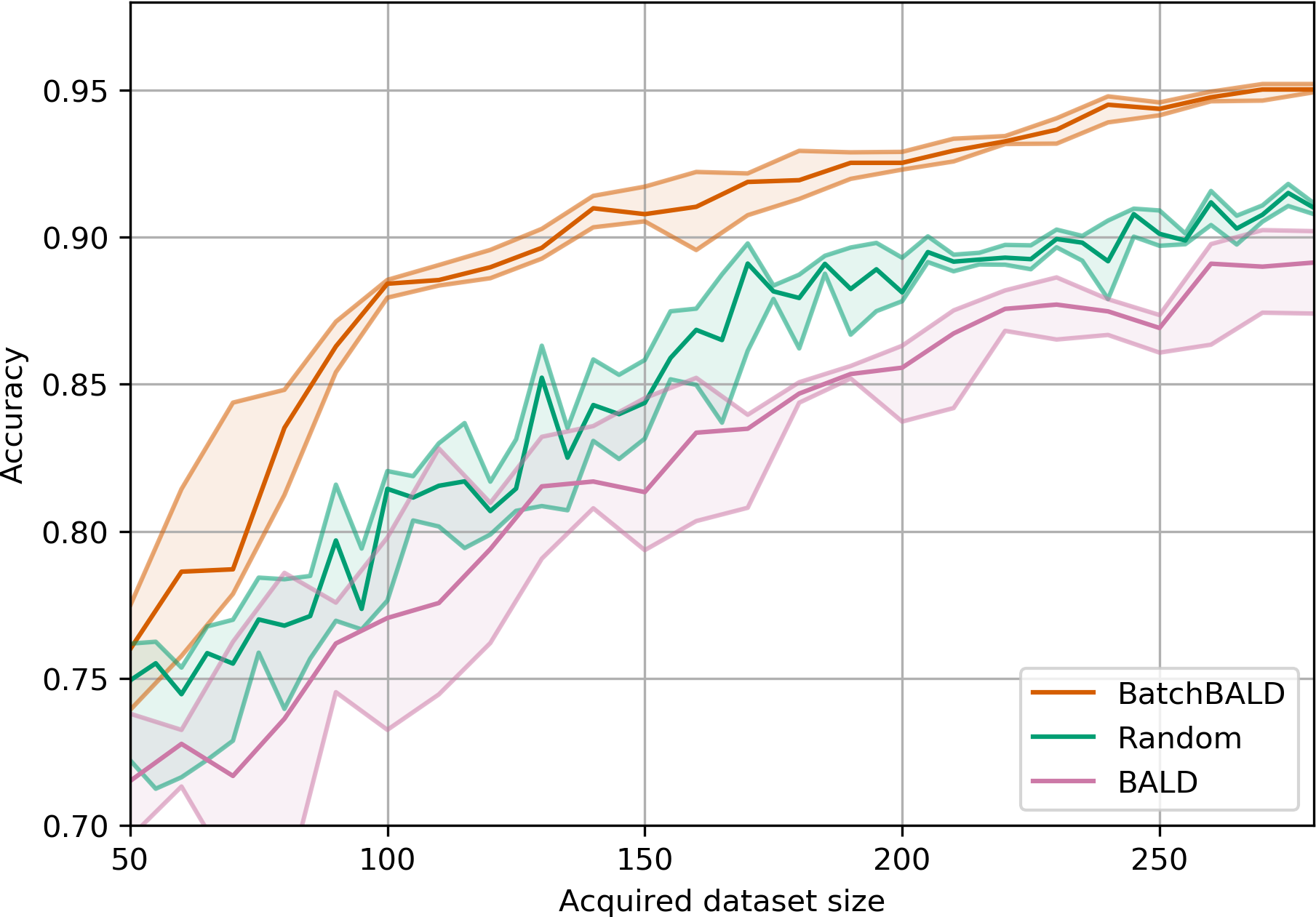}
		\caption{\emph{Performance on \emph{Repeated MNIST} with acquisition size 10}.
		See section \ref{repeated_mnist} for further details. BatchBALD
		outperforms BALD while BALD performs worse than random
		acquisition due to the replications in the dataset.}
		\label{rmnist_graph}
	\end{minipage}
\end{figure}

Naively finding the best batch to acquire requires enumerating all possible
subsets within the available data, which is intractable as the number of
potential subsets grows exponentially with the acquisition size $\numB$ and the
size of available points to choose from. Instead, we develop a greedy algorithm
that selects a batch in linear time, and show that it is at worst a $1 -
\nicefrac{1}{e}$ approximation to the optimal choice for our acquisition
function. We provide an open-source
implementation\footnote{\url{https://github.com/BlackHC/BatchBALD}}.

The main contributions of this work are:
\begin{enumerate}
	\item \emph{BatchBALD}, a data-efficient active learning method that acquires \emph{sets} of high-dimensional image
	data, leading to improved data efficiency and reduced total run time, section \ref{batchbald};
	\item a greedy algorithm to select a batch of points efficiently, section \ref{batchbald_algorithm}; and
    \item an estimator for the acquisition function that scales to larger acquisition sizes and to datasets with
	many classes, section \ref{batchbald_derivation}.
\end{enumerate}

\section{Background}

\subsection{Problem Setting}
The Bayesian active learning setup consists of an unlabelled dataset $\Dpool$,
the current training set $\Dtrain$, a Bayesian model $\Model$ with model
parameters $\w \sim \probc{\w}{\Dtrain}$, and output predictions $\probc{\y}{\x, \w,
\Dtrain}$ for data point $\x$ and prediction $\y \in \left \{ 1, ..., \numC
\right \}$ in the classification case. The conditioning of $\w$ on $\Dtrain$
expresses that the model has been trained with $\Dtrain$. Furthermore, an oracle
can provide us with the correct label $\tilde{y}$ for a data point in the
unlabelled pool $\x \in \Dpool$. The goal is to obtain a certain level of
prediction accuracy with the least amount of oracle queries. At each acquisition
step, a batch of data points $\left \{ \sxbs \right \}$ is selected using an acquisition function $\theAF$ which scores a
candidate batch of unlabelled data points $\left \{ \xbs \right \} \subseteq
\Dpool$ using the current model parameters $\probc{\w}{\Dtrain}$:
\begin{equation}
    \left \{ \sxbs \right \} = \argmax_{ \left \{ \xbs \right \} \subseteq \Dpool} {\theAF \left ( \left \{ \xbs \right \}, \probc{\w}{\Dtrain} \right )}
    .
\end{equation}

\subsection{BALD}
BALD (\emph{Bayesian Active Learning by Disagreement}) \citep{houlsby2011bayesian}
uses an acquisition function that estimates the mutual information between the
model predictions and the model parameters. Intuitively, it captures how
strongly the model predictions for a given data point and the model parameters
are coupled, implying that finding out about the true label of data points with
high mutual information would also inform us about the true model parameters.
Originally introduced outside the context of deep learning, the only
requirement on the model is that it is Bayesian. BALD is defined as:
\begin{align}
    \label{eq:BALD}
    \MIc{\y}{\w}{\x, \Dtrain} = \Hc{\y}{\x, \Dtrain} - \E{\Hc{\y}{\x, \w, \Dtrain}}{\probc{\w}{\Dtrain}}.
\end{align}
Looking at the two terms in equation \eqref{eq:BALD}, for the
mutual information to be high, the left term has to be high and the right term
low. The left term is the entropy of the model prediction, which is high when
the model's prediction is uncertain. The right term is an expectation of the
entropy of the model prediction over the posterior of the model parameters and
is low when the model is overall certain for each draw of model parameters from
the posterior. Both can only happen when the model has many possible ways to
explain the data, which means that the posterior draws are disagreeing among
themselves.

BALD was originally intended for acquiring individual data points and immediately
retraining the model. This becomes a bottleneck in deep learning, where
retraining takes a substantial amount of time. Applications of
BALD \citep{gal2016dropout, janz2017actively} usually acquire the top $\numB$.
This can be expressed as summing over individual scores:
\begin{align}
	\BALDAF \left ( \left \{ \xbs \right \}, \probc{\w}{\Dtrain} \right ) = \sum_{\ii=1}^{\numB} \MIc{\y_i}{\w}{\x_i, \Dtrain},
\end{align}
and finding the optimal batch for this acquisition function using a greedy algorithm,
which reduces to picking the top $\numB$ highest-scoring data points.

\subsection{Bayesian Neural Networks (BNN)}
In this paper we focus on BNNs as our Bayesian model because they scale well to high
dimensional inputs, such as images. Compared to regular neural networks, BNNs
maintain a distribution over their weights instead of point estimates.
Performing exact inference in BNNs is intractable for any reasonably sized
model, so we resort to using a variational approximation. Similar to
\citet{gal2017deep}, we use MC dropout \citep{gal2016dropout}, which is easy to
implement, scales well to large models and datasets, and is straightforward to
optimise.

\section{Methods}

\subsection{BatchBALD}
\label{batchbald}

We propose \emph{BatchBALD} as an extension of BALD whereby we jointly score points by estimating the mutual information between a \emph{joint of multiple data points} and the model parameters:\footnote{
	We use the notation $\MIc{x, y}{ z}{c}$ to denote the mutual information between the \emph{joint of the random variables} $x, y$ and the random variable $z$ conditioned on $c$.
	}
\begin{align}
	\BatchBALDAF \left ( \left \{ \xbs \right \}, \probc{\w}{\Dtrain} \right ) = \MIc{\ybs }{ \w}{\xbs, \Dtrain}.
\end{align}
This builds on the insight that independent selection of a batch of data points
leads to data inefficiency as correlations between data points in an acquisition
batch are not taken into account.

To understand how to compute the mutual information between a set of points and the
model parameters, we express $\xbs$, and $\ybs$ through joint random
variables $\xbv$ and $\ybv$ in a product probability space and use the definition of the mutual information for two random variables:
\begin{align}
    \MIc{\ybv}{\w}{\xbv, \Dtrain} = \Hc{\ybv}{\xbv, \Dtrain} - \chainedE{\Hc{\ybv}{\xbv, \w, \Dtrain}}{\probc{\w}{\Dtrain}}.
\end{align}
Intuitively, the mutual information between two random variables can be seen as the intersection of their information content.
In fact, \citet{yeung1991new} shows that a signed measure $\mu^*$ can be defined for discrete random variables $x$, $y$, such that
$\MI{x}{y} = \im{x \cap y}$, $\Entropy{x,y} = \im{x \cup y}$, $\chainedE{\Hc{x}{y}}{\prob{y}} = \im{x \setminus y}$, and so on,
where we identify random variables with their counterparts in information space, and conveniently drop conditioning on $\Dtrain$ and $\x_i$.

Using this, BALD can be viewed as the sum of individual intersections $\sum_\ii \im{\yi \cap \w}$,
which double counts overlaps between the $\yi$. Naively extending BALD to the mutual information between $\ybs$ and $\w$,
which is equivalent to $\im{\bigcap_\ii \yi \cap \w}$, would lead to selecting \emph{similar} data points instead of diverse ones under maximisation.

BatchBALD, on the other hand, takes overlaps into account by computing $\im{\bigcup_\ii \yi \cap \w}$
and is more likely to acquire a more diverse cover under maximisation:
\begin{align}
	& \MIc{\ybs }{ \w}{\xbs, \Dtrain} = \Hc{\ybv}{\xbv, \Dtrain} - \chainedE{\Hc{\ybv}{\xbv, \w, \Dtrain}}{\probc{\w}{\Dtrain}} \\
	= & \im{\bigcup_\ii \yi} - \im{\bigcup_\ii \yi \setminus \w} = \im{\bigcup_\ii \yi \cap \w}
\end{align}
This is depicted in figure \ref{im_intuition} and also motivates that $\BatchBALDAF \le \BALDAF$, which we prove in appendix \ref{bald_approximates_batchbald}.
For acquisition size 1, BatchBALD and BALD are equivalent.

\begin{figure}[t]
	\begin{subfigure}[t]{.49\textwidth}%
		\centering
		\includegraphics[width=0.65\linewidth]{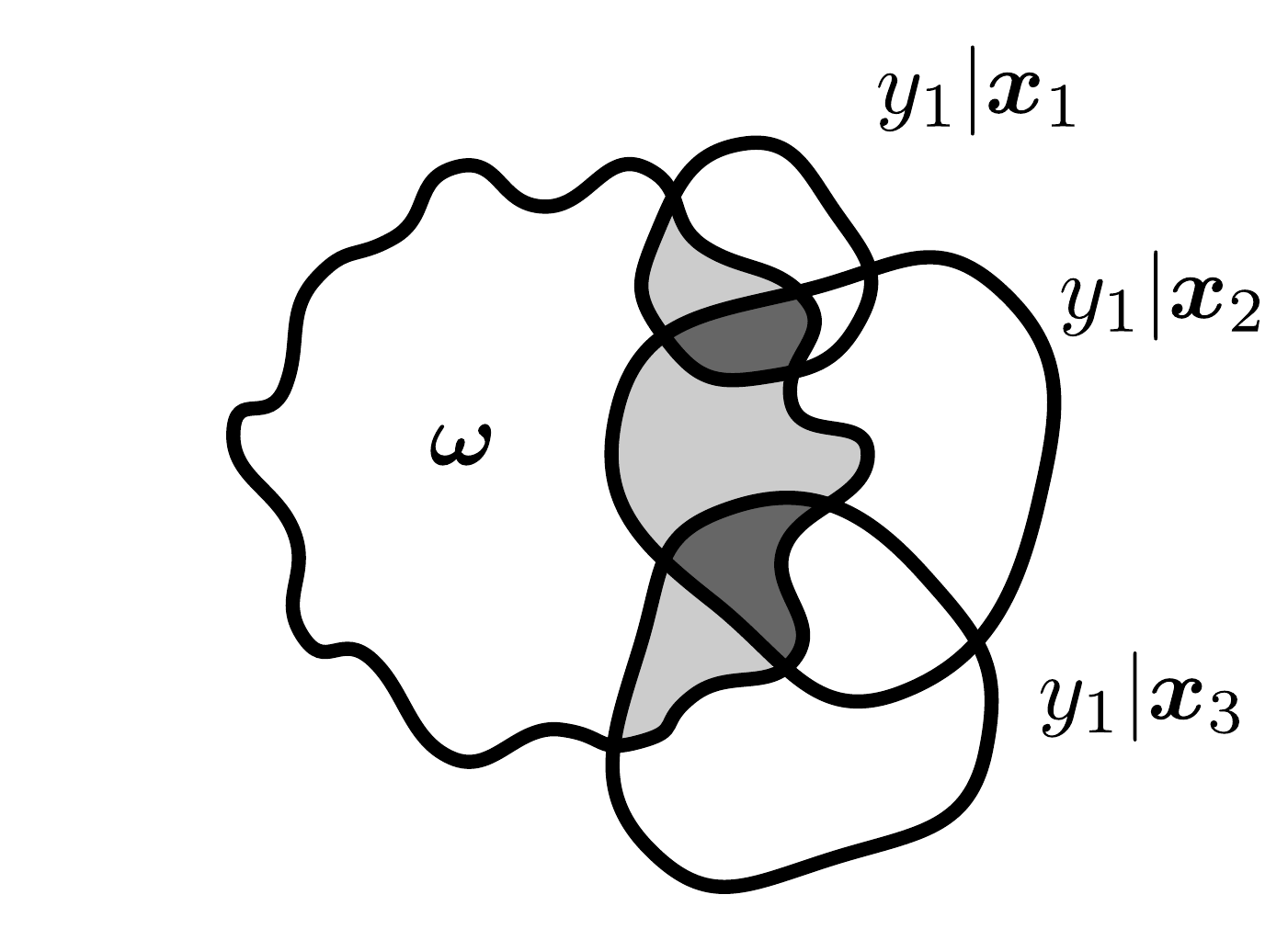} %
	\end{subfigure}%
	\begin{subfigure}[t]{.49\textwidth}%
		\centering
		\includegraphics[width=0.65\linewidth]{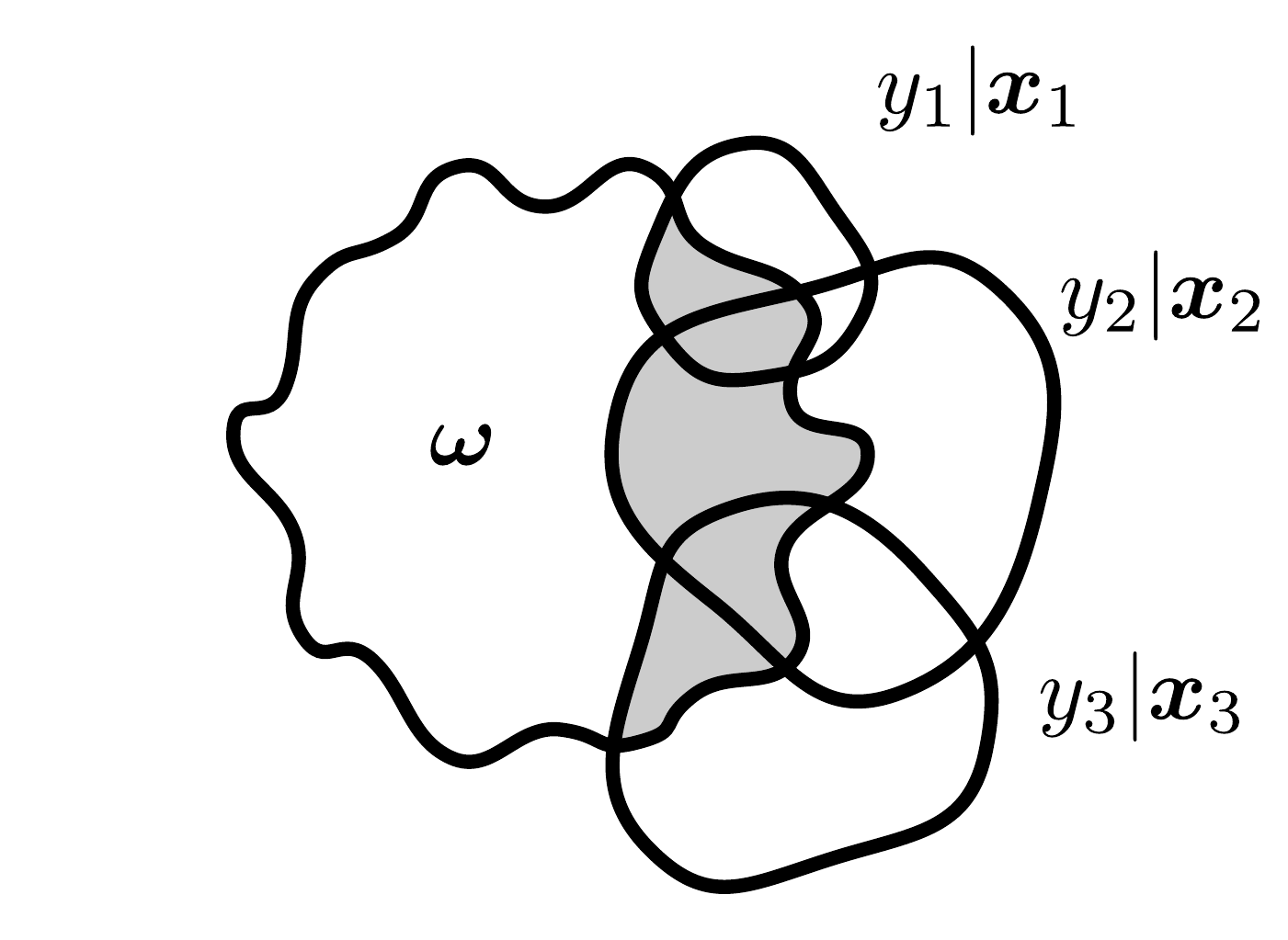}
	\end{subfigure}
	\begin{subfigure}[b]{.49\textwidth}%
		$$\sum_{\ii} \MIc{\y_i}{\w}{\x_i, \Dtrain} = \sum_\ii \im{\yi \cap \w}$$
		\caption{\textbf{BALD}}
		\label{BALD_intuition}
	\end{subfigure}%
	\begin{subfigure}[b]{.49\textwidth}%
		$$\MIc{\ybs }{ \w}{\xbs, \Dtrain} = \im{\bigcup_\ii \yi \cap \w}$$
		\caption{\textbf{BatchBALD}}
		\label{BatchBALD_intuition}
	\end{subfigure}%
	\caption{\emph{Intuition behind \emph{BALD} and \emph{BatchBALD} using I-diagrams \citep{yeung1991new}.}
	\emph{BALD} overestimates the joint mutual information. \emph{BatchBALD}, however, takes the overlap between variables into account and will strive to acquire a better cover of $\w$.
	Areas contributing to the respective score are shown in grey, and areas that are double-counted in dark grey.}
	\label{im_intuition}
\end{figure}

\subsection{Greedy approximation algorithm for BatchBALD}
\label{batchbald_algorithm}

\begin{algorithm}[t]
	\caption{Greedy BatchBALD $1 - \nicefrac{1}{e}$-approximate algorithm \label{algo:greedy_batchbald}}

    \DontPrintSemicolon
    \SetAlgoLined

	\KwIn{acquisition size $\numB$, unlabelled dataset $\Dpool$, model parameters $\probc{\w}{\Dtrain}$}

	$A_0 \leftarrow \emptyset$ \;
	\For { $n \leftarrow 1$ \KwTo $\numB$}{
		\lForEach
			{$\x \in \Dpool \setminus A_{n-1}$}
			{$\displaystyle s_{\x} \leftarrow
			\BatchBALDAF \left ( A_{n-1} \cup \left \{ \x \right \}, \probc{\w}{\Dtrain} \right )$
		} \label{code:batch_bald_computation}

	 	$\displaystyle \x_{n} \leftarrow \argmax_{ \x \in \Dpool \setminus A_{n-1} } s_{\x} $ \;
	 	$\displaystyle A_n \leftarrow A_{n-1} \cup \left \{ \x_n \right \}$	\;
	}
	\KwOut{ acquisition batch $A_n = \left \{ \xbs \right \}$ }
\end{algorithm}

To avoid the combinatorial explosion that arises from jointly scoring subsets of
points, we introduce a greedy approximation for computing BatchBALD, depicted in
algorithm \ref{algo:greedy_batchbald}. In appendix \ref{submodular_proof}, we
prove that $\BatchBALDAF$ is submodular, which means the greedy algorithm is
$1-\nicefrac{1}{e}$-approximate \citep{nemhauser1978analysis, krause2008near}.

In appendix \ref{batchbald_equivalence}, we show that, under idealised conditions,
when using BatchBALD and a fixed final $\left| \Dtrain \right |$,
the active learning loop itself can be seen as a greedy $1-\nicefrac{1}{e}$-approximation algorithm,
and that an active learning loop with BatchBALD and acquisition size larger than 1 is bounded by an
an active learning loop with individual acquisitions,
that is BALD/BatchBALD with acquisition size 1, which is the ideal case.

\subsection{Computing $\BatchBALDAF$}
\label{batchbald_derivation}
\newcommand{\pyxh}{\probc{\ybv}{\xbv}}
\newcommand{\pyxha}{\probc{\ybiv}{\xbiv}}
\newcommand{\qyxa}{\qprobc{\ybi}{\xbi}}

\newcommand{\pyxaw}{\probc{\ybi}{\xbi \w}}
\newcommand{\pyxhw}{\probc{\ybv}{\xbv \w}}

\newcommand{\si}{s}
\newcommand{\ybvrs}{\hat{\y}_{1:\numN, s}}

\newcommand{\ybvars}{\hat{\y}_{1:\numN+1, s}}
\newcommand{\yar}{\hat{\y}_{\numN+1}}

\newcommand{\yr}{\hat{\y}}
\newcommand{\ynr}{\yr_{\numN}}
\newcommand{\ynvr}{\yr_{1:\numN}}
\newcommand{\yndvr}{\yr_{1:\numN-1}}
\newcommand{\pwfuchsia}{{\leavevmode\color{Sepia}\prob{\w}}}
{
	\renewcommand{\pw}{\pwfuchsia}

For brevity, we leave out conditioning on $\xns$, and $\Dtrain$, and $\pw$ denotes $\probc{\w}{\Dtrain}$ in this section.
$\BatchBALDAF$ is then written as:
\begin{align}
	\label{eq:BB_simple}
    \BatchBALDAF \left ( \left \{ \xns \right \}, \pw \right ) = \Entropy{\yns} - \E{\Hc{\yns}{\w}}{\pw}.
\end{align}
Because the $\y_i$ are independent when conditioned on $\w$, computing the right
term of equation \eqref{eq:BB_simple} is simplified as the conditional joint
entropy decomposes into a sum. We can approximate the expectation using a
Monte-Carlo estimator with $\numK$ samples from our model parameter distribution $\wj \sim \pw$:
\begin{equation}
    \E{\Hc{\yns}{\w}}{\pw} = \sum_{\ii=1}^{\numN}\E{\Hc{\yi}{\w}}{\pw} \approx \frac{1}{\numK} \sum_{\ii=1}^{\numN} \sum_{\jj=1}^{\numK} \Hc{\yi}{\wj}.
\end{equation}
Computing the left term of equation \eqref{eq:BB_simple} is difficult because
the unconditioned joint probability does not factorise. Applying the equality
$\prob{\y} = \E{\probc{\y}{\w}}{\pw}$, and, using sampled $\wj$, we
compute the entropy by summing over all possible configurations $\ynvr$ of
$\ynv$:
\begin{align}
    \Entropy{\yns} &= \E{-\log{\prob{\yns}}}{\prob{\yns}} \\
    &=
    	\chainedE{
			\E{-\log{
    			\E{\probc{\yns}{\w}}{\pw}
    		}}
    		{\probc{\yns}{\w}}
    	}
    	{\pw} \\
	&\approx -\sum_{\ynvr} \left ( \frac{1}{\numK} \sum_{\jj=1}^{\numK} \probc{\ynvr}{\wj} \right ) \log \left ( {
		\frac{1}{\numK} \sum_{\jj=1}^{\numK} \probc{\ynvr}{\wj} } \right ). \label{eq:bb_approx1}
\end{align}

\subsection{Efficient implementation}
In each iteration of the algorithm, $\xnds$ stay fixed
while $\xn$ varies over $\Dpool \setminus A_{n-1}$. We can reduce the required
computations by factorizing $\probc{\ynv}{\w}$ into $\probc{\yndv}{\w}
\probc{\yn}{\w}$. We store $ \probc{\yndvr}{\wj}$ in a matrix
$\hat{P}_{1:\numN-1}$ of shape $\numC^{\numN-1} \times \numK$ and $\probc{\yn}{\wj}$
in a matrix $\hat{P}_{\numN}$ of shape $\numC \times \numK$. The sum $\sum_{\jj=1}^{\numK}
\probc{\ynvr}{\wj}$ in \eqref{eq:bb_approx1} can be then be turned into a matrix product:
\begin{equation}
	\frac{1}{\numK} \sum_{\jj=1}^{\numK} \probc{\ynvr}{\wj} =
	\frac{1}{\numK} \sum_{\jj=1}^{\numK} \probc{\yndvr}{\wj} \probc{\ynr}{\wj} =  \left ( \frac{1}{\numK} \hat{P}_{1:\numN-1} \hat{P}_{\numN}^T \right )_{\yndvr, \ynr}.
\end{equation}
This can be further sped up by using batch matrix multiplication to compute the
joint entropy for different $\xn$. $\hat{P}_{1:\numN-1}$ only has to be
computed once, and we can recursively compute $\hat{P}_{1:\numN}$ using
$\hat{P}_{1:\numN-1}$ and $\hat{P}_{\numN}$, which allows us to sample $\probc{\y}{\wj}$
 for each $\x \in \Dpool$ only once at the beginning of the algorithm.

\newcommand{\cDpool}{{| \Dpool |}}

For larger acquisition sizes, we use $\numM$ MC samples of $\yndv$ as enumerating all
possible configurations becomes infeasible. See appendix \ref{batchbald_mc_approx} for details.

Monte-Carlo sampling bounds the time complexity of the full BatchBALD algorithm to $\mathcal{O} (\numB \numC \cdot \min \{ \numC^\numB, \numM \} \cdot \cDpool \cdot \numK)$ compared to $\mathcal{O} (\numC^\numB \cdot \cDpool^\numB \cdot \numK)$ for naively finding the exact optimal batch and
$\mathcal{O} ( \left ( \numB + \numK \right ) \cdot \cDpool)$ for BALD\footnote{
	$\numB$ is the acquisition size, $\numC$ is the number of classes,
	$\numK$ is the number of MC dropout samples, and
	$\numM$ is the number of sampled configurations of $\yndv$.
	}.
}

\section{Experiments}
\begin{figure}[t]
	\begin{subfigure}[t]{.49\textwidth}%
		\centering
		\includegraphics[width=0.95\linewidth]{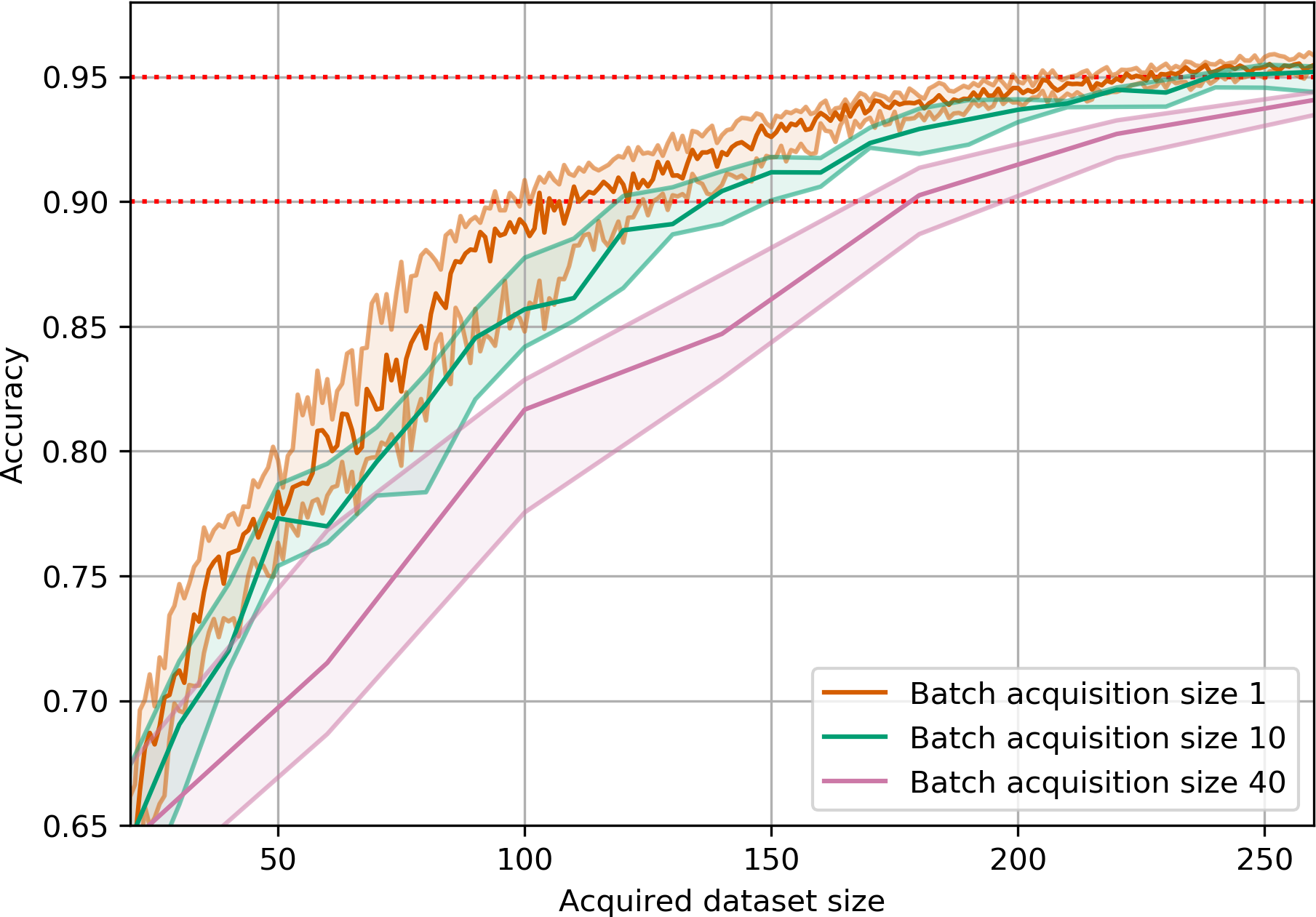} %
		\caption{\textbf{BALD}}
		\label{MNIST_BALD}
	\end{subfigure}%
	\hfill
	\begin{subfigure}[t]{.49\textwidth}%
		\centering
		\includegraphics[width=0.95\linewidth]{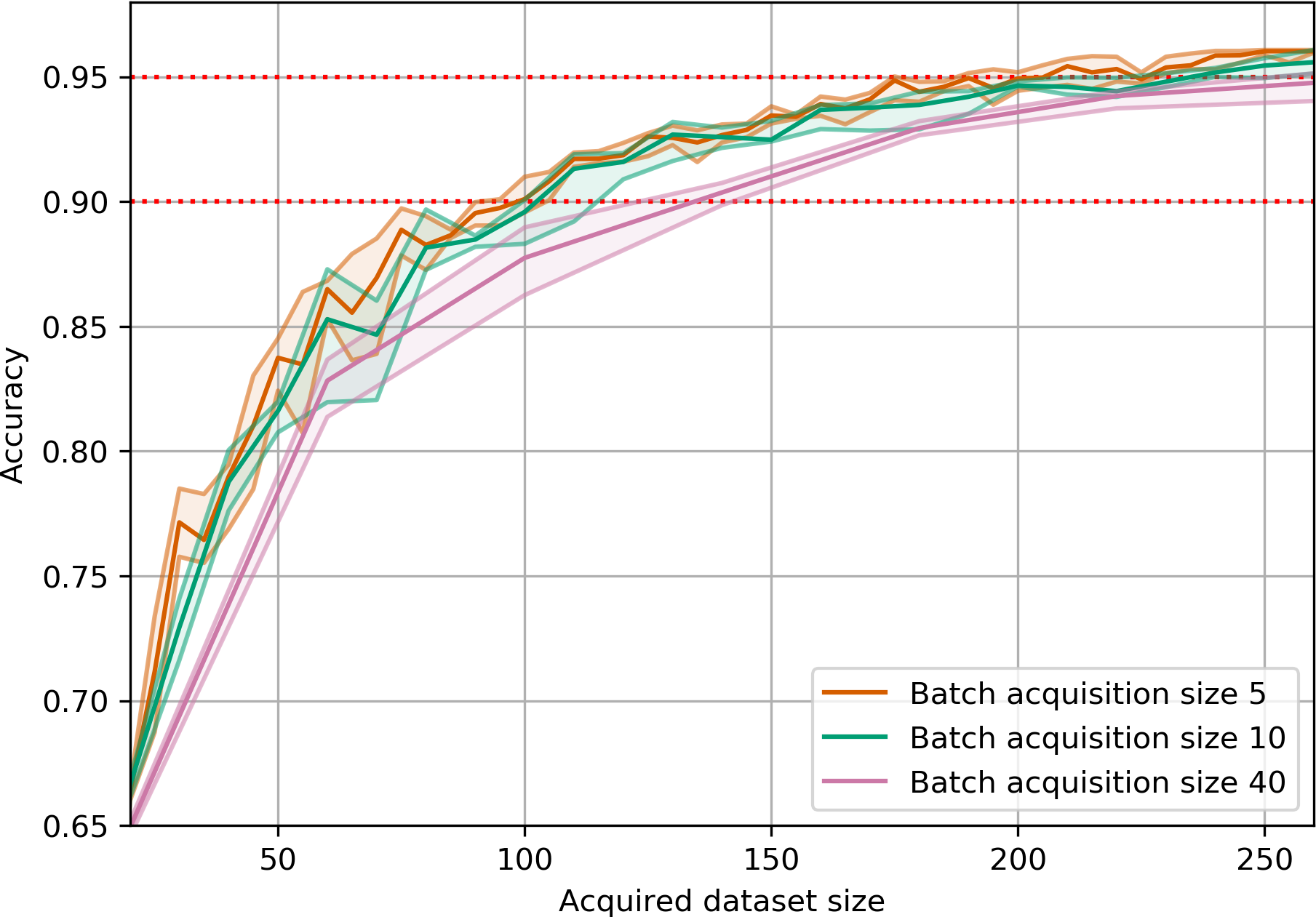}
		\caption{\textbf{BatchBALD}}
		\label{MNIST_BatchBALD}
	\end{subfigure}%
	\caption{\emph{Performance on \emph{MNIST} for increasing acquisition sizes.}
	BALD's performance drops drastically as the acquisition size increases.
	BatchBALD maintains strong performance even with increasing acquisition size.}
\end{figure}

In our experiments, we start by showing how a naive application of the BALD
algorithm to an image dataset can lead to poor results in a dataset with many
(near) duplicate data points, and show that BatchBALD solves this
problem in a grounded way while obtaining favourable results (figure
\ref{rmnist_graph}).

We then illustrate BatchBALD's effectiveness on standard AL datasets: MNIST and
EMNIST. EMNIST \citep{cohen2017emnist} is an extension of MNIST that also
includes letters, for a total of 47 classes, and has a twice as large training
set. See appendix \ref{emnist_visualisation} for examples of the dataset. We
show that BatchBALD provides a substantial performance improvement in these
scenarios, too, and has more diverse acquisitions. Finally, we look at BatchBALD
in the setting of transfer learning, where we finetune a large pretrained model
on a more difficult dataset called CINIC-10 \citep{darlow2018cinic}, which is a
combination of CIFAR-10 and downscaled ImageNet.

In our experiments, we repeatedly go through active learning loops. One active
learning loop consists of training the model on the available labelled data and
subsequently acquiring new data points using a chosen acquisition function. As the
labelled dataset is small in the beginning, it is important to avoid
overfitting. We do this by using early stopping after 3 epochs of declining
accuracy on the validation set. We pick the model with the highest validation
accuracy. Throughout our experiments, we use the Adam \citep{kingma2014adam}
optimiser with learning rate 0.001 and betas 0.9/0.999. All our results report the median
of 6 trials, with lower and upper quartiles. We use these quartiles to draw
the filled error bars on our figures.

We reinitialize the model after each acquisition, similar to
\citet{gal2017deep}. We found this helps the model improve even when very small
batches are acquired. It also decorrelates subsequent acquisitions as
final model performance is dependent on a particular
initialization \citep{frankle2018lottery}.

When computing $\probc{y}{x, \w, \Dtrain}$, it is important to keep the dropout
masks in MC dropout consistent while sampling from the model. This is
necessary to capture dependencies between the inputs for BatchBALD, and it makes
the scores for different points more comparable by removing this source of
noise. We do not keep the
masks fixed when computing BALD scores because its performance usually benefits
from the added noise. We also do not need to keep these masks fixed for training
and evaluating the model.

In all our experiments, we either compute joint entropies exactly by enumerating
all configurations, or we estimate them using 10,000 MC samples, picking
whichever method is faster. In practice, we compute joint entropies exactly for roughly the
first 4 data points in an acquisition batch and use MC sampling thereafter.
\begin{figure}[tbp]
	\begin{minipage}[t]{0.49\textwidth}
		\centering
		\includegraphics[width=0.95\linewidth]{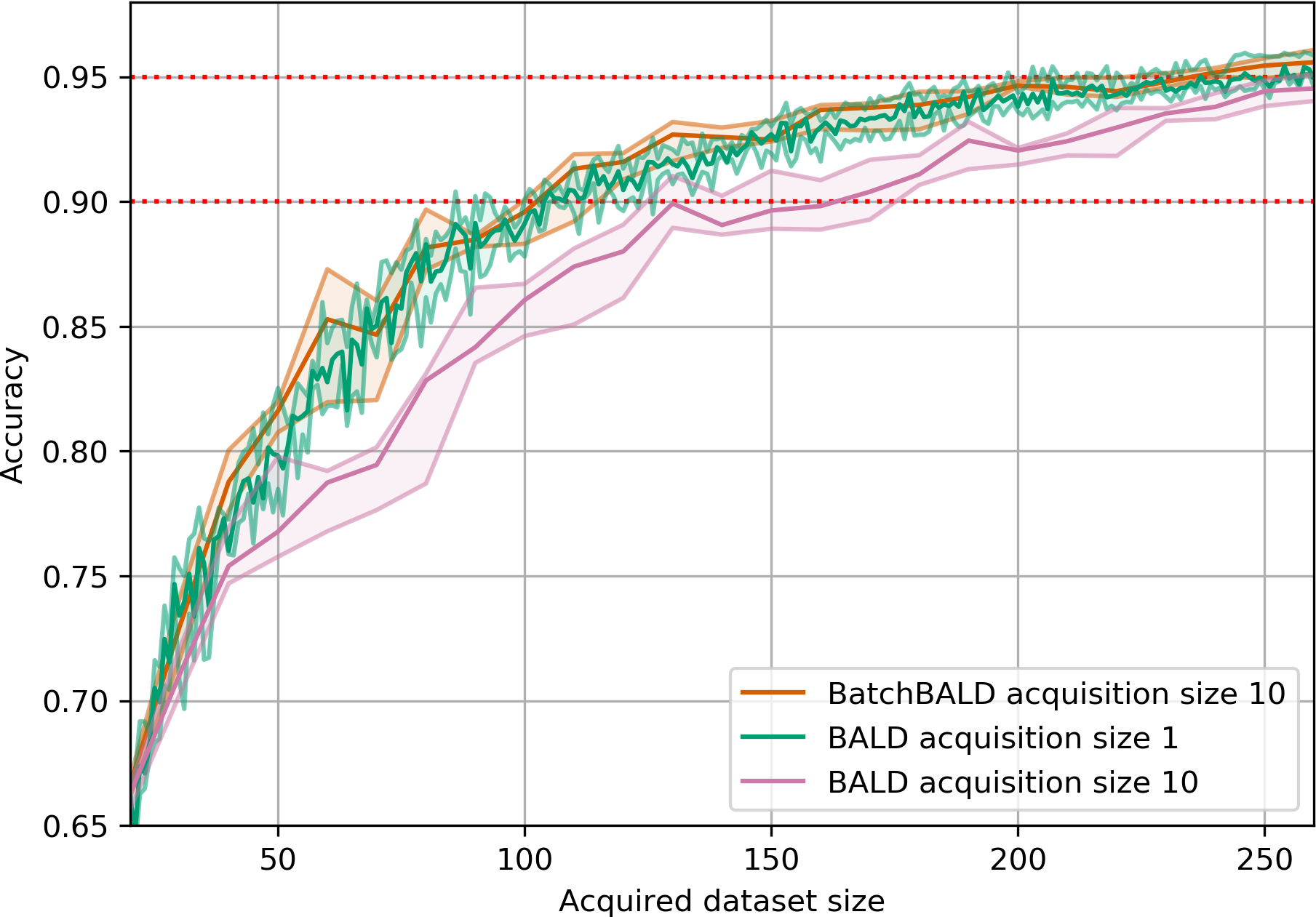}%
		\caption{\emph{Performance on \emph{MNIST}.}
		BatchBALD outperforms BALD with acquisition size 10 and performs close to
		the optimum of acquisition size 1.
		}
		\label{MNIST_BALD_MULTIBALD}
	\end{minipage}
	\hfill
	\begin{minipage}[t]{0.49\textwidth}
		\includegraphics[width=0.95\linewidth]{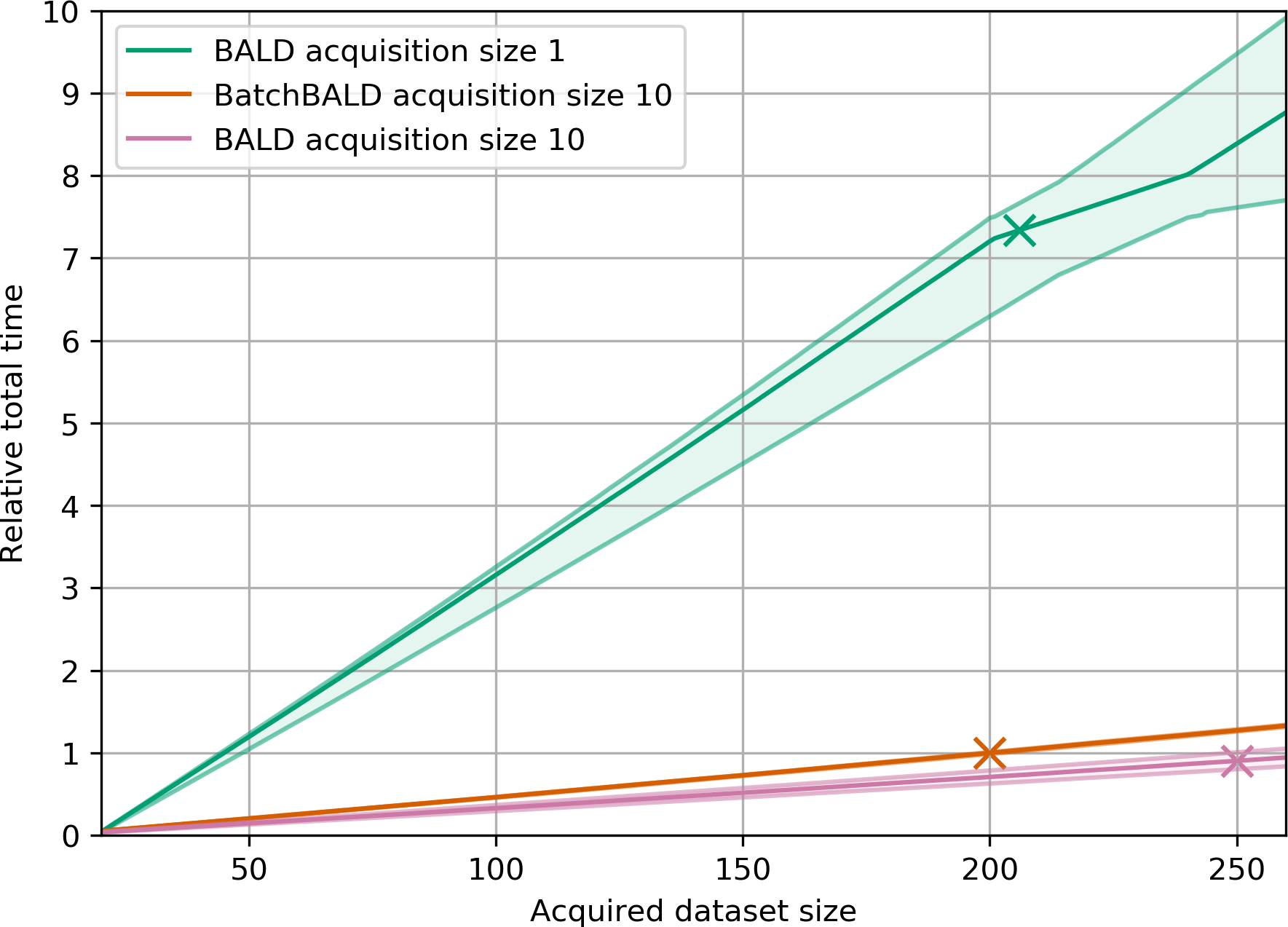}
	   \caption{\emph{Relative total time on \emph{MNIST}.}
	   Normalized to training BatchBALD with acquisition size 10 to 95\% accuracy. The stars mark when 95\%
	   accuracy is reached for each method.
	   }
	   \label{time_mnist}
	\end{minipage}
\end{figure}

\subsection{Repeated MNIST}
\label{repeated_mnist}
As demonstrated in the introduction, naively applying BALD to a dataset that
contains many (near) replicated data points leads to poor performance. We show
how this manifests in practice by taking the MNIST dataset and replicating each
data point in the training set two times (obtaining a training set that is three
times larger than the original). After normalising the dataset, we add isotropic
Gaussian noise with a standard deviation of 0.1 to simulate slight differences
between the duplicated data points in the training set. All results are obtained
using an acquisition size of 10 and 10 MC dropout samples. The initial dataset
was constructed by taking a balanced set of 20 data points\footnote{These
initial data points were chosen by running BALD 6 times with the initial dataset
picked randomly and choosing the set of the median model. They were subsequently
held fixed.}, two of each class (similar to \citep{gal2017deep}).

Our model consists of two blocks of [convolution, dropout, max-pooling, relu],
with 32 and 64 5x5 convolution filters. These blocks are followed by a two-layer
MLP that includes dropout between the layers and has 128 and 10 hidden units.
The dropout probability is 0.5 in all three locations. This architecture
achieves 99\% accuracy with 10 MC dropout samples during test time on the full MNIST dataset.

The results can be seen in figure \ref{rmnist_graph}. In this illustrative
scenario, BALD performs poorly, and even randomly acquiring points performs
better. However, BatchBALD is able to cope with the replication perfectly. In
appendix \ref{ablation_rmnist}, we look at varying the repetition number and
show that as we increase the number of repetitions BALD gradually performs
worse. In appendix \ref{rmnist_comparison_graph}, we also compare with Variation
Ratios \citep{freeman1965elementary}, and Mean STD \citep{kendall2015bayesian}
which perform on par with random acquisition.

\subsection{MNIST}

\begin{table}[t]
	\centering
	\caption{\emph{Number of required data points on \emph{MNIST} until 90\% and 95\% accuracy are reached.} 25\%-, 50\%-
	and 75\%-quartiles for the number of required data points when available.}
	\label{tb:mnist}
	\begin{tabular}{l l l}
			  & 90\% accuracy & 95\% accuracy \\ \hline
	BatchBALD & 70 / 90 / 110     & 190 / 200 / 230                \\
	BALD \footnotemark & 120 / 120 / 170   & 250 / 250 / \textgreater{}300  \\
	BALD \citep{gal2017deep} &   145   &  335
	\end{tabular}
\end{table}
\footnotetext{reimplementation using reported experimental setup}

For the second experiment, we follow the setup of \citet{gal2017deep} and perform
AL on the MNIST dataset using
100 MC dropout samples. We use the same model architecture and initial dataset
as described in section \ref{repeated_mnist}.
Due to differences in model architecture, hyper parameters and model retraining, we
significantly outperform the original results in \citet{gal2017deep} as shown in table \ref{tb:mnist}.

We first look at BALD for increasing acquisition size in figure
\ref{MNIST_BALD}. As we increase the acquisition size from the ideal of
acquiring points individually and fully retraining after each points
(acquisition size 1) to 40, there is a substantial performance drop.

BatchBALD, in figure \ref{MNIST_BatchBALD}, is able to maintain performance when
doubling the acquisition size from 5 to 10. Performance drops only slightly at
40, possibly due to estimator noise.

The results for acquisition size 10 for both BALD and BatchBALD are compared in
figure \ref{MNIST_BALD_MULTIBALD}.
BatchBALD outperforms BALD.
Indeed, BatchBALD with acquisition size 10 performs close to the ideal with acquisition size
1.
The total run time of training these three models until 95\%
accuracy is visualized in figure \ref{time_mnist}, where we see that BatchBALD with acquisition size 10
is much faster than BALD with acquisition size 1, and only marginally slower
than BALD with acquisition size 10.

\subsection{EMNIST}
In this experiment, we show that BatchBALD also provides a significant
improvement when we consider the more difficult EMNIST dataset
\citep{cohen2017emnist} in the \emph{Balanced} setup, which consists of 47
classes, comprising letters and digits. The training set consists of 112,800
28x28 images balanced by class, of which the last 18,800 images constitute the validation set.
We do not use an initial dataset and instead perform
the initial acquisition step with the randomly initialized model. We use 10 MC dropout samples.

We use a similar model architecture as before, but with added capacity. Three
blocks of [convolution, dropout, max-pooling, relu], with 32, 64 and 128 3x3
convolution filters, and 2x2 max pooling. These blocks are followed by a
two-layer MLP with 512 and 47 hidden units, with again a dropout layer in
between. We use dropout probability 0.5 throughout the model.

The results for acquisition size 5 can be seen in figure \ref{EMNIST_BALD_MULTIBALD}. BatchBALD
outperforms both random acquisition and BALD while BALD is unable to beat
random acquisition. Figure
\ref{entropy_labels_emnist} gives some insight into why BatchBALD performs better than
BALD. The entropy of the categorical distribution of acquired class labels is consistently
higher, meaning that BatchBALD acquires a more diverse set of data points. In
figure \ref{histogram_labels_emnist}, the classes on the x-axis are sorted by
number of data points that were acquired of that class. We see that BALD
undersamples classes while BatchBALD is more consistent.

\begin{figure}[t]
	\begin{minipage}[t]{0.49\textwidth}
		\centering
		\includegraphics[width=0.95\linewidth]{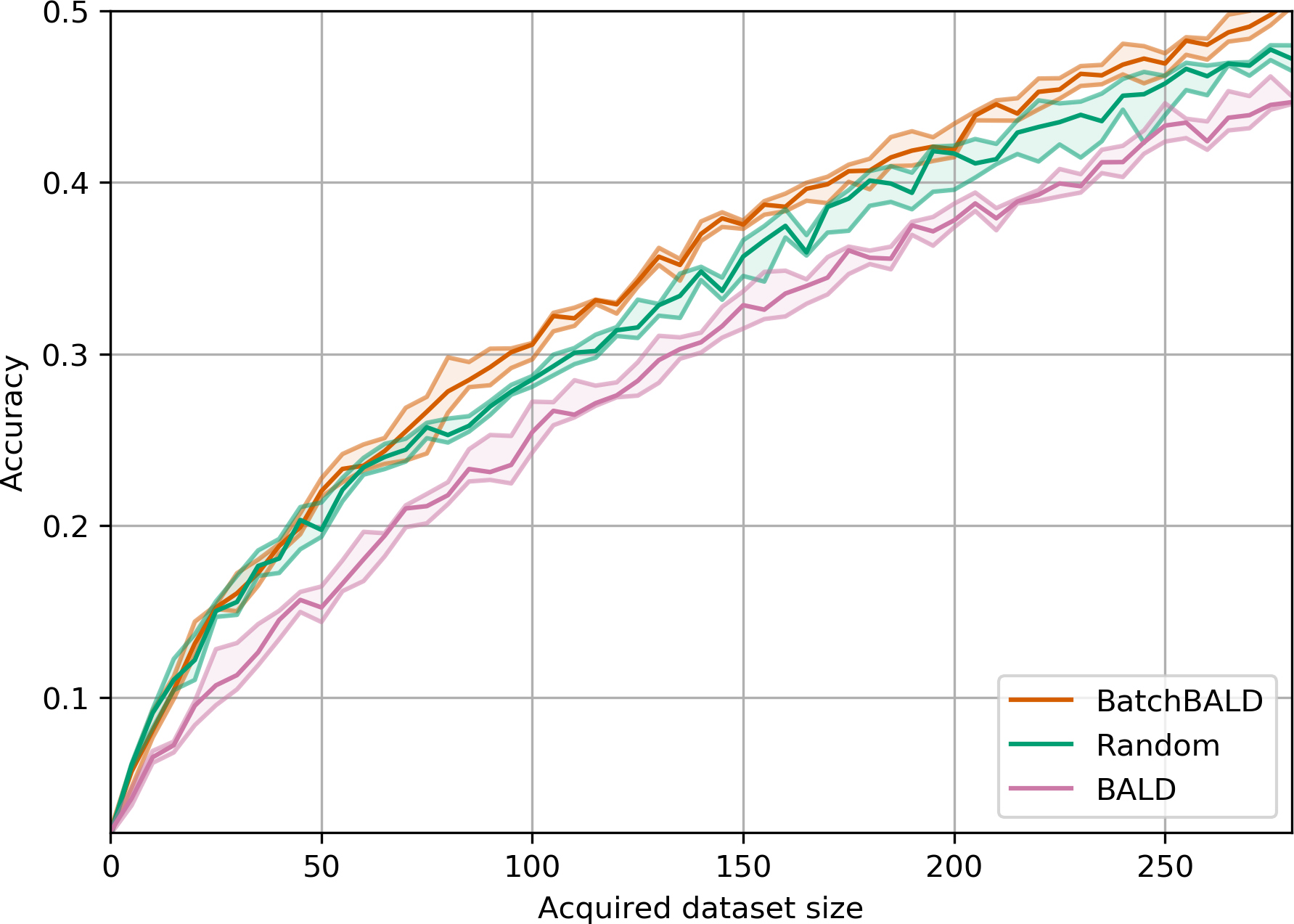}
		\caption{\emph{Performance on \emph{EMNIST}.}
		BatchBALD consistently outperforms both random acquisition and BALD
		while BALD is unable to beat random acquisition.
		}
		\label{EMNIST_BALD_MULTIBALD}
	\end{minipage}
	\hfill
	\begin{minipage}[t]{0.49\textwidth}
		\centering
        \includegraphics[width=0.99\linewidth]{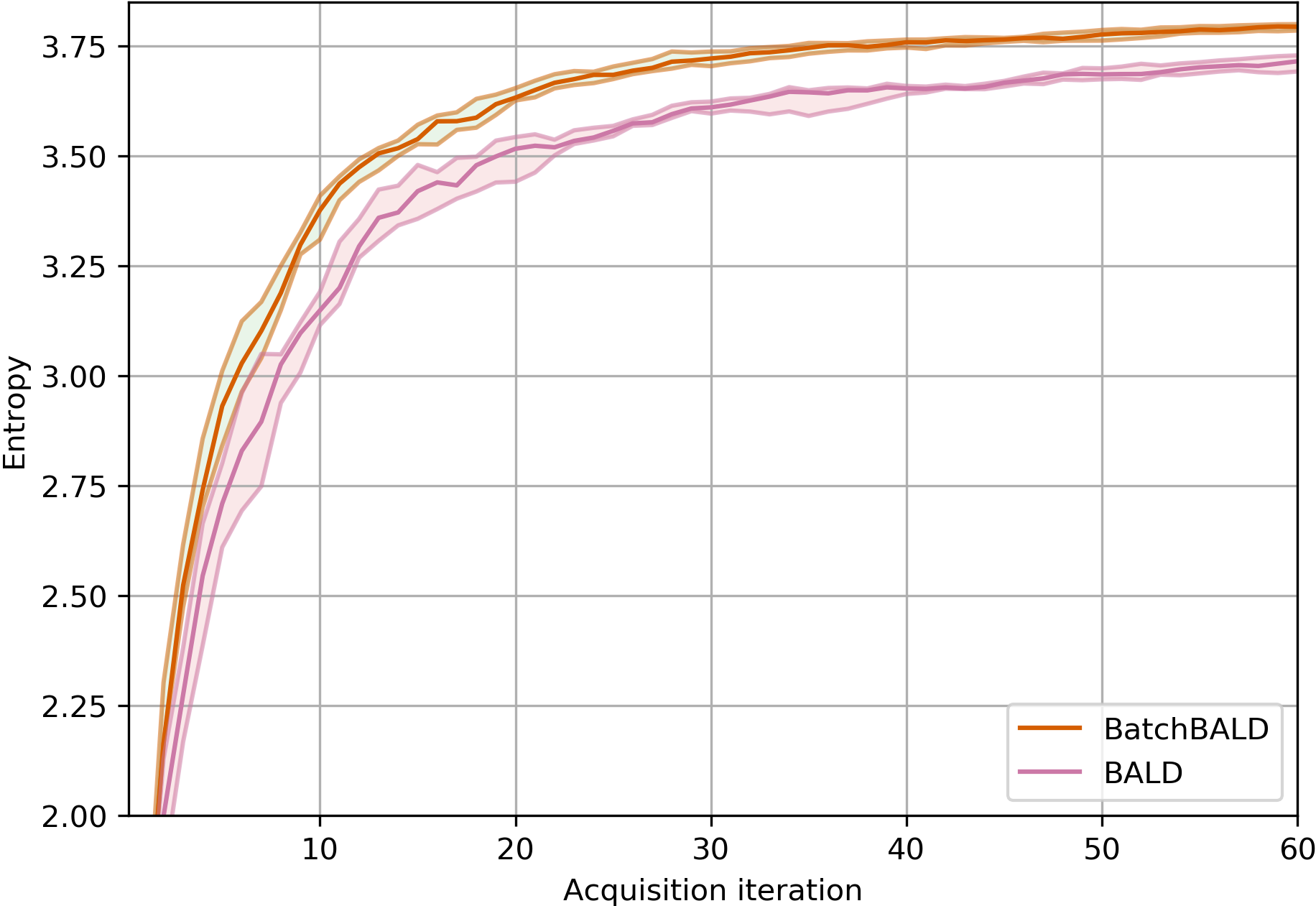}
		\caption{\emph{Entropy of acquired class labels over acquisition steps on \emph{EMNIST}.}
		BatchBALD steadily acquires a more diverse set of data points.
		}
        \label{entropy_labels_emnist}
	\end{minipage}
\end{figure}

\begin{wrapfigure}{HLR}{0.45\textwidth}
	\centering
	\includegraphics[width=0.95\linewidth]{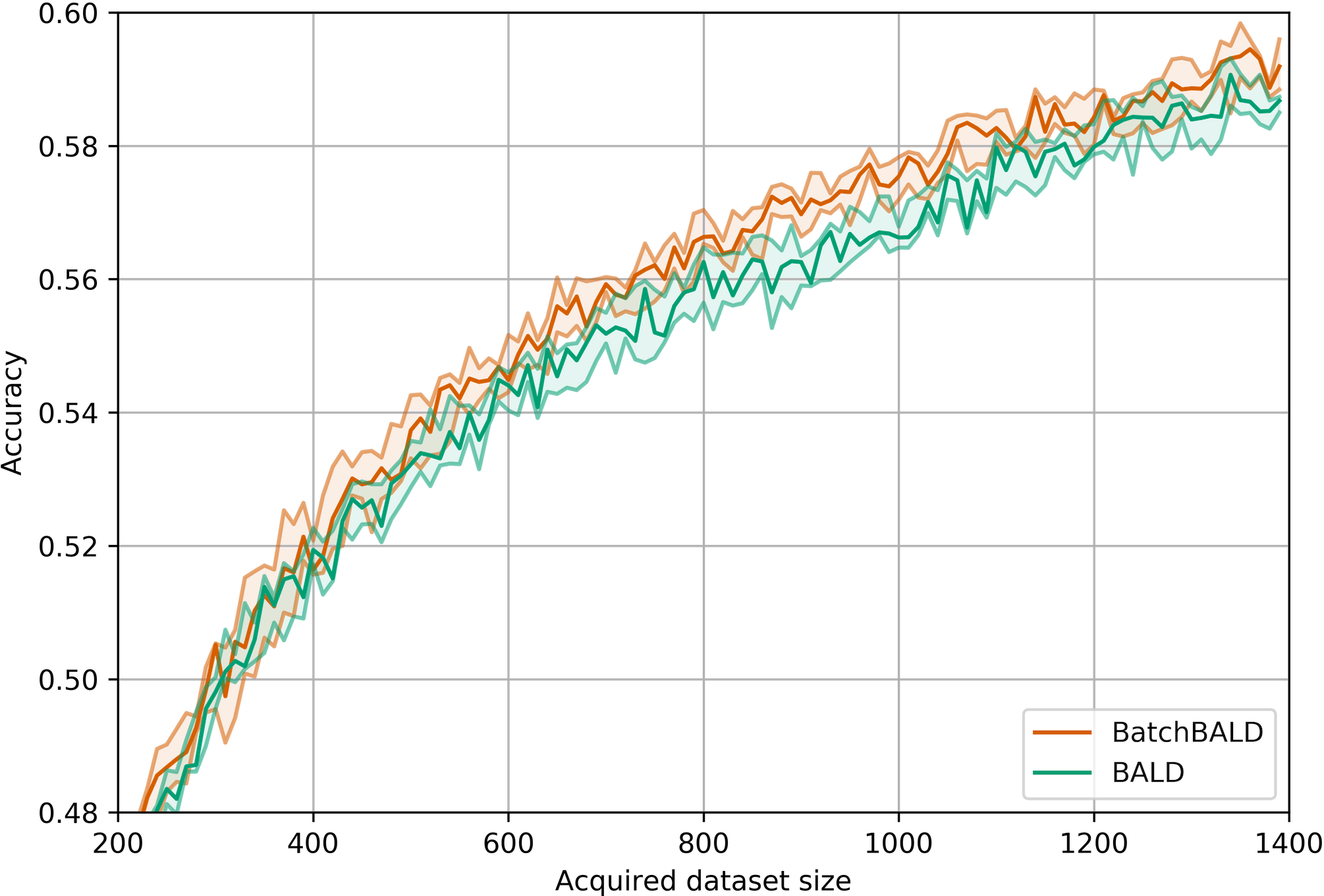}
	\caption{\emph{Performance on \emph{CINIC-10}.}
	BatchBALD outperforms BALD from 500 acquired samples onwards.}
	 \vspace{-1.5em}
    \label{cinic10_results}
\end{wrapfigure}

\subsection{CINIC-10}
CINIC-10 is an interesting dataset because it is large (270k data points) and
its data comes from two different sources: CIFAR-10 and ImageNet. To get strong
performance on the test set it is important to obtain data from both sets.
Instead of training a very deep model from scratch on a small dataset, we opt to
run this experiment in a transfer learning setting, where we use a pretrained
model and acquire data only to finetune the original model. This is common
practice and suitable in cases where data is abound for an auxiliary domain, but
is expensive to label for the domain of interest.

For the CINIC-10 experiment, we use 160k training samples for the unlabelled
pool, 20k validation samples, and the other 90k as test samples. We use an
ImageNet pretrained VGG-16, provided by PyTorch \citep{paszke2017automatic}, with a
dropout layer before a 512 hidden unit (instead of 4096) fully connected layer.
We use 50 MC dropout samples, acquisition size 10 and repeat the experiment for
6 trials. The results are in figure \ref{cinic10_results}, with the 59\% mark reached at
1170 for BatchBALD and 1330 for BALD (median).

\section{Related work}
AL is closely related to Bayesian Optimisation (BO), which is concerned with
finding the global optimum of a function \citep{snoek2012practical}, with the
fewest number of function evaluations. This is generally done using a Gaussian
Process. A common problem in BO is the lack of parallelism, with usually a
single worker being responsible for function evaluations. In real-world
settings, there are usually many such workers available and making optimal use
of them is an open problem \citep{gonzalez2016batch,alvi2019asynchronous} with
some work exploring mutual information for optimising a multi-objective problem
\citep{hernandez2016predictive}.

Maintaining diversity when acquiring a batch of data has also been attempted
using constrained optimisation \citep{guo2008discriminative} and in Gaussian
Mixture Models \citep{azimi2012batch}. In AL of molecular data, the lack of
diversity in batches of data points acquired using the BALD objective has been
noted by \citet{janz2017actively}, who propose to resolve it by limiting the
number of MC dropout samples and relying on noisy estimates.

A related approach to AL is semi-supervised learning (also sometimes referred to
as weakly-supervised), in which the labelled data is commonly assumed to be
fixed and the unlabelled data is used for unsupervised learning
\citep{kingma2014semi,rasmus2015semi}. \citet{wang2017cost, sener2017active,
sinha2019variational} explore combining it with AL.

\section{Scope and limitations}
\textbf{Unbalanced datasets} BALD and BatchBALD do not work well when the test
set is unbalanced as they aim to learn well about all classes and do not follow
the density of the dataset. However, if the test set is balanced, but the
training set is not, we expect BatchBALD to perform well.

\textbf{Unlabelled data} BatchBALD does not take into account any information
from the unlabelled dataset. However, BatchBALD uses the underlying Bayesian
model for estimating uncertainty for unlabelled data points, and semi-supervised
learning could improve these estimates by providing more information about the
underlying structure of the feature space. We leave a semi-supervised extension
of BatchBALD to future work.

\textbf{Noisy estimator} A significant amount of noise is introduced by
MC-dropout's variational approximation to training BNNs. Sampling of the joint
entropies introduces additional noise. The quality of larger acquisition batches
would be improved by reducing this noise.

\section{Conclusion}
We have introduced a new batch acquisition function, BatchBALD, for Deep Bayesian
Active Learning, and a greedy algorithm that selects good candidate batches compared to
the intractable optimal solution. Acquisitions show
increased diversity of data points and improved performance over BALD and other
methods.

While our method comes with additional computational cost during acquisition,
BatchBALD is able to significantly reduce the number of data points that need to
be labelled and the number of times the model has to be retrained, potentially
saving considerable costs and filling an important gap in practical Deep
Bayesian Active Learning.

\newpage

\subsubsection*{Acknowledgements}
The authors want to thank Binxin (Robin) Ru for helpful references to
submodularity and the appropriate proofs. We would also like to thank the rest
of OATML for their feedback at several stages of the project. AK is supported by
the UK EPSRC CDT in Autonomous Intelligent Machines and Systems (grant reference
EP/L015897/1). JvA is grateful for funding by the EPSRC (grant reference
EP/N509711/1) and Google-DeepMind. Funding for computational resources was
provided by the Allan Turing Institute and Google.

\subsubsection*{Author contributions}

AK derived the original estimator, proved submodularity and bounds, implemented BatchBALD efficiently, and ran the experiments.
JvA developed the narrative and experimental design, advised on debugging, structured the paper into its current form, and pushed it forward at difficult times.
JvA and AK wrote the paper jointly.

\newpage

\bibliographystyle{plainnat}
\bibliography{references}

\newpage

\appendix

\section{Proof of submodularity}
\label{submodular_proof}

\citet{nemhauser1978analysis} show that if a function is submodular, then a
greedy algorithm like algorithm \ref{algo:greedy_batchbald} is $1-\nicefrac{1}{e}$-approximate. Here, we show that $\BatchBALDAF$ is submodular.

We will show that $\BatchBALDAF$ satisfies the following equivalent definition of submodularity:

\newcommand{\pa}{y_1}
\newcommand{\pb}{y_2}

\begin{definition}
	A function $f$ defined on subsets of $\Omega$ is called \emph{submodular} if for every set $A \subset \Omega$ and two non-identical points $\pa, \pb \in \Omega \setminus A$:
	\begin{equation}
		f(A \cup \{\pa\}) + f(A \cup \{\pb\}) \ge f(A \cup \{\pa, \pb\}) + f(A) \label{eq:submod_ie}
	\end{equation}
\end{definition}

Submodularity expresses that there are "diminishing returns" for adding additional points to $f$.

\begin{lemma}
	$\BatchBALDAF(A, \pw) := \MI{A}{ \w}$ is submodular for $A \subset \Dpool$.
\end{lemma}
\begin{proof}
	\newcommand{\MIw}[1]{\MI{#1}{\w}}
	\newcommand{\MIcw}[2]{\MIc{#1}{\w}{#2}}
	Let $\pa, \pb \in \Dpool, \pa \ne \pb$.
	We start by substituting the definition of $\BatchBALDAF$ into \eqref{eq:submod_ie} and subtracting $\MI{A}{ \w}$ twice on both sides, using that $\MI{A \cup B}{ \w} - \MI{B}{ \w} = \MIc{A }{ \w}{B}$:
	\begin{align}
		&\MIw{A \cup \{y\}} + \MIw{A \cup \{x\}} \ge \MIw{A \cup \{x, y\}} + \MIw{A} \\
	\Leftrightarrow \; &\MIcw{y}{A} + \MIcw{x}{A} \ge \MIcw{x, y}{A}.
	\end{align}
	We rewrite the left-hand side using the definition of the mutual information $\MI{A}{B} = \Entropy{A} - \Hc{A}{B}$ and reorder:
	\begin{align}
		&\MIcw{y}{A} + \MIcw{x}{A} \\
		=
		&\underbrace{\Hc{\pa}{A} + \Hc{\pa}{A}}_{\ge \Hc{\pa, \pb}{A}}
		-
		\underbrace{\left(\Hc{\pa}{A, \w} + \Hc{\pb}{A, \w}\right)}_{= \Hc{\pa, \pb}{A, \w}}
		\\
		\ge &\Hc{\pa, \pb}{A} - \Hc{\pa, \pb}{A, \w}\\
		=
		&\MIcw{x, y}{A}
		,
	\end{align}
	where we have used that entropies are subadditive in general and additive given $\pa \Perp \pb \; \vert \; \w$.
\end{proof}

Following \citet{nemhauser1978analysis}, we can conclude that algorithm \ref{algo:greedy_batchbald} is $1 - \nicefrac{1}{e}$-approximate.

\section{Connection between BatchBALD and BALD}
\label{batchbald_bald_connection}
In the following section, we show that BALD approximates BatchBALD and that
BatchBALD approximates BALD with acquisition size 1. The BALD score is an upper
bound of the BatchBALD score for any candidate batch. At the same time,
BatchBALD can be seen as performing BALD with acquisition size 1 during each
step of its greedy algorithm in an idealised setting.

\subsection{BALD as an approximation of BatchBALD}
\label{bald_approximates_batchbald}
Using the subadditivity of information entropy and the independence of the
$\yi$ given $\w$, we show that BALD is an approximation of BatchBALD
and is always an upper bound on the respective BatchBALD score:
\begin{align}
	&\BatchBALDAF \left ( \left \{ \xbs \right \}, \probc{\w}{\Dtrain} \right ) \\
	= &\Hc{\ybs}{\xbs, \Dtrain} - \E{\Hc{\ybs}{\xbs, \w, \Dtrain}}{\probc{\w}{\Dtrain}} \\
	\le &\sum_{\ii=1}^{\numB} \Hc{\yi}{\x_\ii, \Dtrain} - \sum_{\ii=1}^{\numB} \E{\Hc{\yi}{\x_\ii, \w, \Dtrain}}{\probc{\w}{\Dtrain}} \\
	= &\sum_{\ii=1}^{\numB} \MIc{\yi}{\w}{\x_\ii, \Dtrain} = \BALDAF \left ( \left \{ \xbs \right \}, \probc{\w}{\Dtrain} \right )
\end{align}

\subsection{BatchBALD as an approximation of BALD with acquisition size 1}
\label{batchbald_equivalence}

\newcommand{\matAnd}{A_{n-1}}
\newcommand{\ry}{\tilde{y}}
\newcommand{\rynd}{\ry_{n-1}}
\newcommand{\rynds}{\ry_1, ..., \ry_{n-1}}

To see why BALD with acquisition size 1 can be seen as an upper bound for BatchBALD performance in an idealised setting,
we reformulate line \ref{code:batch_bald_computation} in algorithm \ref{algo:greedy_batchbald} on page \pageref{algo:greedy_batchbald}.
\begin{align}
\shortintertext{
	Instead of the original term $\BatchBALDAF \left ( \matAnd \cup \left \{ \x \right \},
	\probc{\w}{\Dtrain} \right )$, we can equivalently maximise
}
  & \BatchBALDAF \left ( \matAnd \cup \left \{ \x \right \}, \probc{\w}{\Dtrain} \right ) -
\BatchBALDAF \left ( \matAnd, \probc{\w}{\Dtrain} \right) \\
\shortintertext{
	as the right term is
	constant for all $\x \in \Dpool \setminus \matAnd$ within the inner loop, which, in turn, is equivalent to
}
    = & \MIc{\ynds, \y }{\w}{\xnds, \x, \Dtrain} -
	\MIc{\ynds}{\w}{\xndv \Dtrain} \\
	= & \MIc{\y }{\w}{\x, \ynds, \xndv, \Dtrain}
\shortintertext{
	once we expand $\matAnd = \{ \xnds \}$. This means that, at each step of the inner loop, our greedy algorithm is maximising the mutual information of the individual available data points
	with the model parameters conditioned on all the additional data
	points that have already been picked for acquisition and the existing training
	set. Finally, assuming training our model captures all available information,
}
	\ge & \MIc{\y }{\w}{\x, \Dtrain \cup \left\{ (\x_1, \ry_1), ...., (\xnd, \rynd) \right\}} \label{eq:bald_batchbald_eq} \\
	= & \BALDAF \left(
			\left\{ \x \right\},
			\probc{\w}{\Dtrain \cup \left\{ (\x_1, \ry_1), ...., (\xnd, \rynd) \right\}}
		\right),
\end{align}
where $\rynds$ are the actual labels of $\xns$. The mutual information decreases as $\w$ becomes
more concentrated as we expand its training set,
and thus the overlap of $\y$ and $\w$ will become smaller (in an information-measure-theoretical sense).

This shows that every step $n$ of the inner loop in our algorithm is at most as good
as retraining our model on the new training set $\Dtrain \cup \left\{ (\x_1, \ry_1), ...., (\xnd, \rynd) \right\}$
and picking $x_n$ using $\BALDAF$ with acquisition size 1.

\textbf{Relevance for the active training loop.}
We see that the active training loop as a whole is computing a greedy $1 - \nicefrac{1}{e}$-approximation of
the mutual information of all acquired data points over all acquisitions with the model parameters.

\section{Sampling of configurations}
\label{batchbald_mc_approx}
{
	\renewcommand{\pw}{\pwfuchsia}
	We are using the same notation as in section \ref{batchbald_derivation}.
We factor $\probc{\ynv}{\w}$ to avoid recomputations and rewrite
$\Entropy{\ynv}$ as:
\begin{align}
	\Entropy{\ynv} &=
		\chainedE{
			\E{
				-\log{\prob{\ynv}}
			}{\probc{\ynv}{\w}}
		}{\pw} \\
	&=
		\chainedE{
			\E{
				-\log{\prob{\ynv}}
			}{\probc{\yndv}{\w} \probc{\yn}{\w} }
		}{\pw} \\
	&=
		\chainedE{
			\chainedE{
				\E{
					-\log{\prob{\ynv}}
				}{\probc{\yn}{\w}}
			}{\probc{\yndv}{\w}}
		}{\pw}
\end{align}
To be flexible in the way we sample $\yndv$, we perform importance sampling of
$\probc{\yndv}{\w}$ using $\prob{\yndv}$, and, assuming we also have $\numM$
samples $\yndvr$ from $\prob{\yndv}$, we can approximate:
\begin{align}
	& \Entropy{\ynv} =
		\chainedE{
			\E{
				\frac{\probc{\yndv}{\w}}{\prob{\yndv}}
				\E{
					-\log{\prob{\ynv}}
				}{\probc{\yn}{\w}}
			}{\prob{\yndv}}
		}{\pw} \\
	&=
		\chainedE{
			\chainedE{
				\E{
					-\frac{\probc{\yndv}{\w}}{\prob{\yndv}} \log{
						\E{
							\probc{\yndv}{\w} \probc{\ynv}{\w}
						}{\pw}
					}
				}{\probc{\yn}{\w}}
			}{\pw}
		}{\prob{\yndv}} \\
	& \approx
		-\frac{1}{\numM} \sum_{\yndvr}^{\numM}{
			\sum_{\ynr}{
				\frac{
					\frac{1}{\numK} \sum_{\wj}{
						\probc{\yndvr}{\wj} \probc{\ynr}{\wj}
					}
				}{
						\prob{\yndvr}
				}
				\log \left ( {
					\frac{1}{\numK} \sum_{\wj} \probc{\yndvr}{\wj} \probc{\ynr}{\wj}
				} \right )
			}
		} \\
	& =
		-\frac{1}{\numM} \sum_{\yndvr}^{\numM}{
			\sum_{\ynr}{
				\frac{
					\left ( \hat{P}_{1:{\numN-1}} \hat{P}_{\numN}^T \right )_{\yndvr, \ynr}
				}{
					\left ( \hat{P}_{1:{\numN-1}} \mathbbm{1}_{\numK,1} \right )_{\yndvr}
				}
				\log \left( {
					\frac{1}{\numK} \left ( \hat{P}_{1:{\numN-1}} \hat{P}_{\numN}^T \right )_{\yndvr, \ynr}
				} \right )
			}
		}, \label{eq:is_matrix}
\end{align}
where we store $\probc{\yndvr}{\wj}$ in a matrix $\hat{P}_{1:\numN-1}$ of shape
$\numM \times \numK$ and $\probc{\ynr}{\wj}$ in a matrix $\hat{P}_{\numN}$ of
shape $\numC \times \numK$ and $\mathbbm{1}_{\numK,1}$ is a $\numK \times 1 $
matrix of $1$s. Equation \eqref{eq:is_matrix} allows us to cache $\hat{P}_{1:\numN-1}$
inside the inner loop of algorithm \ref{algo:greedy_batchbald} and use batch
matrix multiplication for efficient computation.
}

\newpage

\section{Ablation study on Repeated MNIST}
\label{ablation_rmnist}
To better understand the effect of redundant data points on BALD and BatchBALD,
we run the RMNIST experiment with an increasing number of repetitions. The
results can be seen in figure \ref{ablation_rmnist_figure}. We use the same setup as in
section \ref{repeated_mnist}. BatchBALD performs the same on all repetition
numbers (100 data points till 90\%). BALD achieves 90\% accuracy at 120 data
points (0 repetitions), 160 data points (1 repetition), 280 data points (2
repetitions), 300 data points (4 repetitions). This shows that BALD and
BatchBALD behave as expected.

\begin{figure}[h]
	\centering
	\includegraphics[height=5cm]{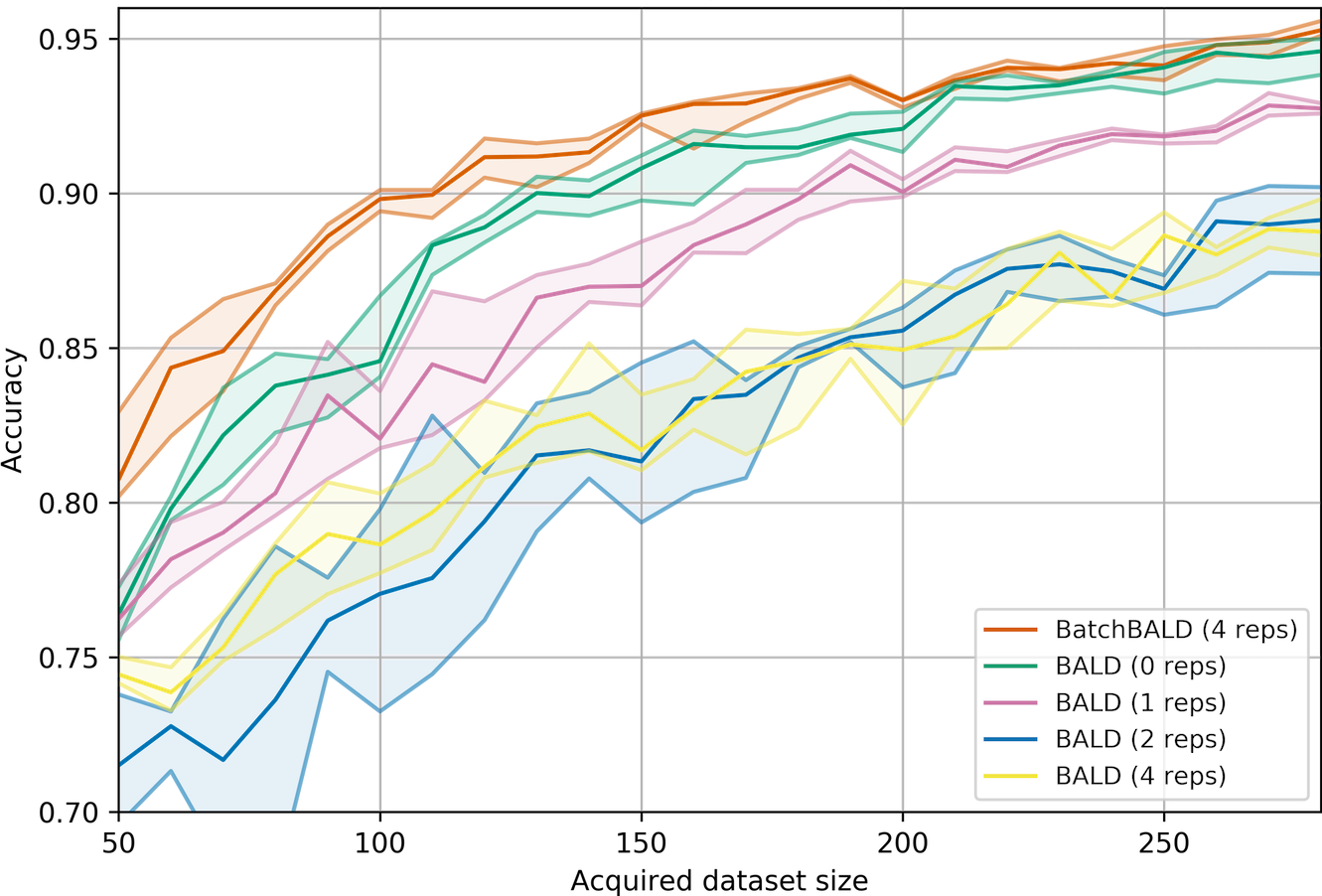}
	\caption{\emph{Performance of \emph{BALD} on \emph{Repeated MNIST} for increasing amount of repetitions.}
	We see that BALD performs worse as the number of repetitions is increased,
	while BatchBALD outperforms BALD with zero repetitions.}
	\label{ablation_rmnist_figure}
\end{figure}

\section{Additional results for Repeated MNIST}
\label{rmnist_comparison_graph}

We show that BatchBALD also outperforms Var Ratios \citep{freeman1965elementary}
and Mean STD \citep{kendall2015bayesian}.

\begin{figure}[h]
	\centering
	\includegraphics[height=5cm]{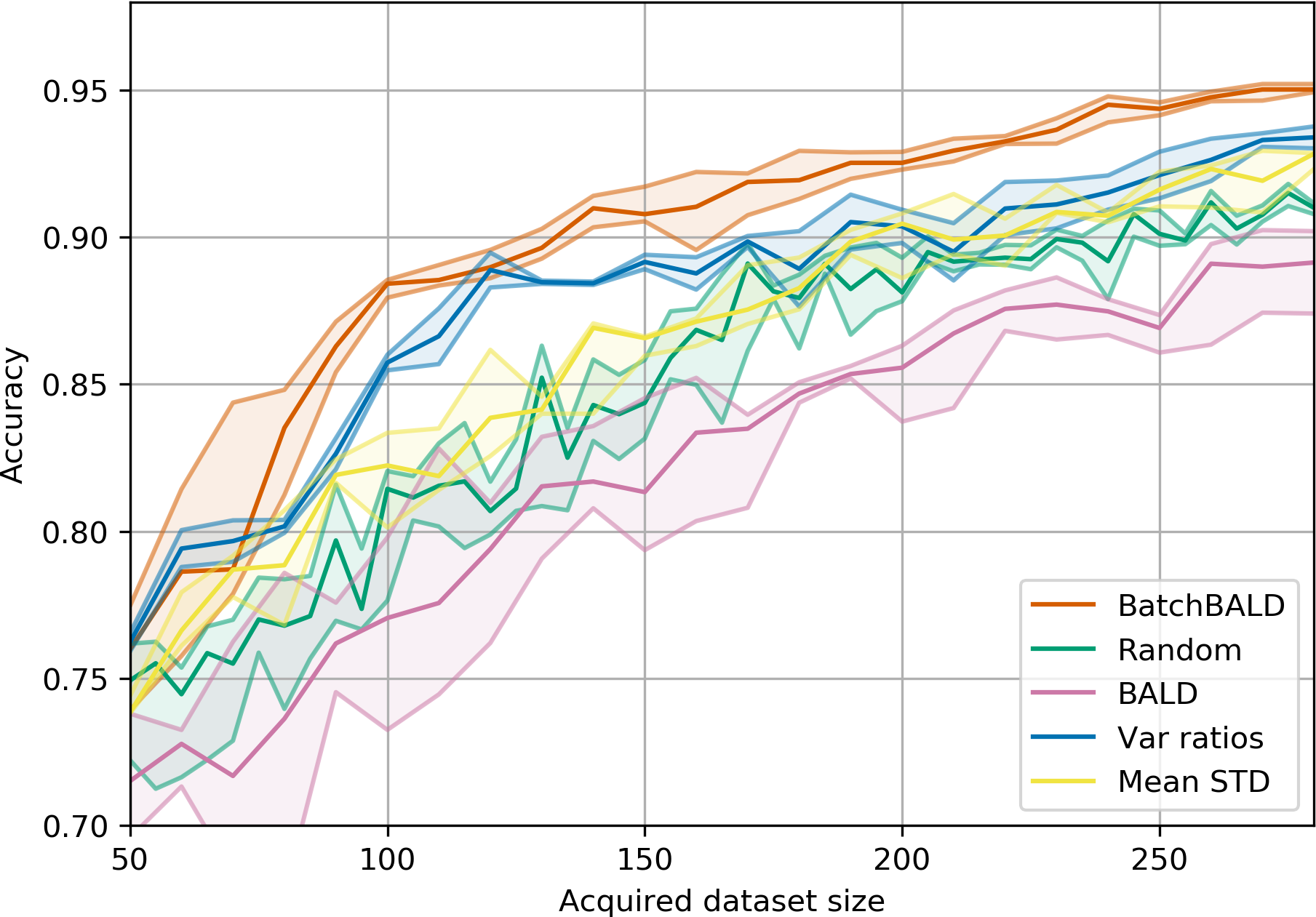}
	\caption{\emph{Performance on \emph{Repeated MNIST}.}
	BALD, BatchBALD, Var Ratios, Mean STD and random acquisition with acquisition size 10 and 10 MC dropout samples.}
\end{figure}

\newpage
\section{Example visualisation of EMNIST}
\label{emnist_visualisation}

\begin{figure}[h]
	\centering
	\includegraphics[height=5cm]{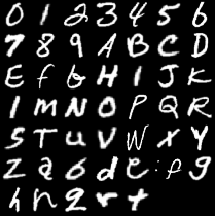}
	\caption{\emph{Examples of all 47 classes of EMNIST}}
\end{figure}

\section{Entropy and class acquisitions including random acquisition}
\label{emnist_random_diversity}
\begin{figure}[h]
	\begin{minipage}[t]{0.49\textwidth}
		\centering
		\includegraphics[width=0.95\linewidth]{images/EMNIST_zero_initial_data.png}
		\caption{\emph{Performance on \emph{EMNIST}.}
		BatchBALD consistently outperforms both random acquisition and BALD
		while BALD is unable to beat random acquisition.
		}
	\end{minipage}
	\hfill
	\begin{minipage}[t]{0.49\textwidth}
		\centering
        \includegraphics[width=0.99\linewidth]{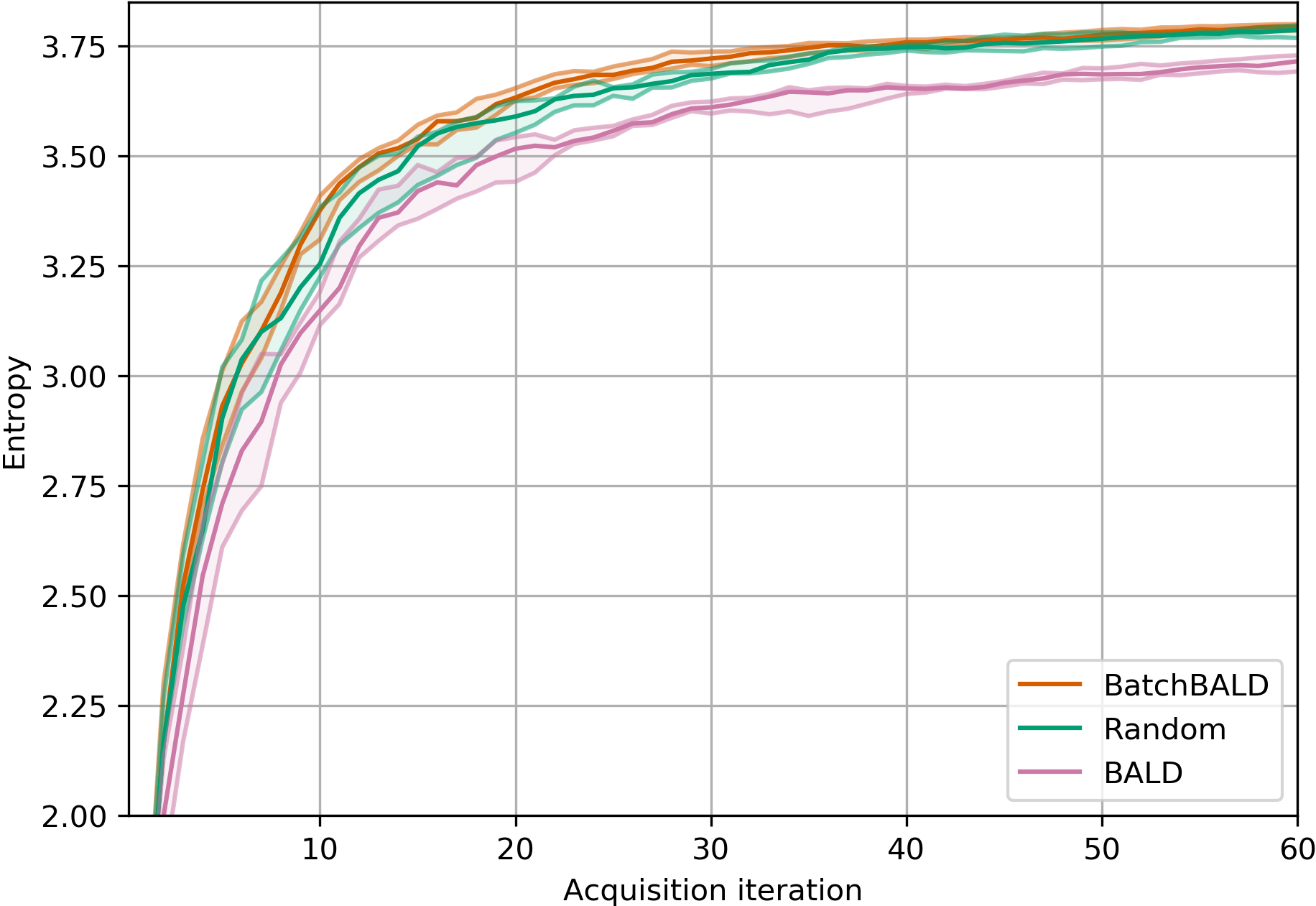}
		\caption{\emph{Entropy of acquired class labels over acquisition steps on \emph{EMNIST}.}
		BatchBALD steadily acquires a more diverse set of data points than BALD.
		}
	\end{minipage}
\end{figure}
\begin{figure}[h]%
	\centering
	\includegraphics[width=0.99\linewidth]{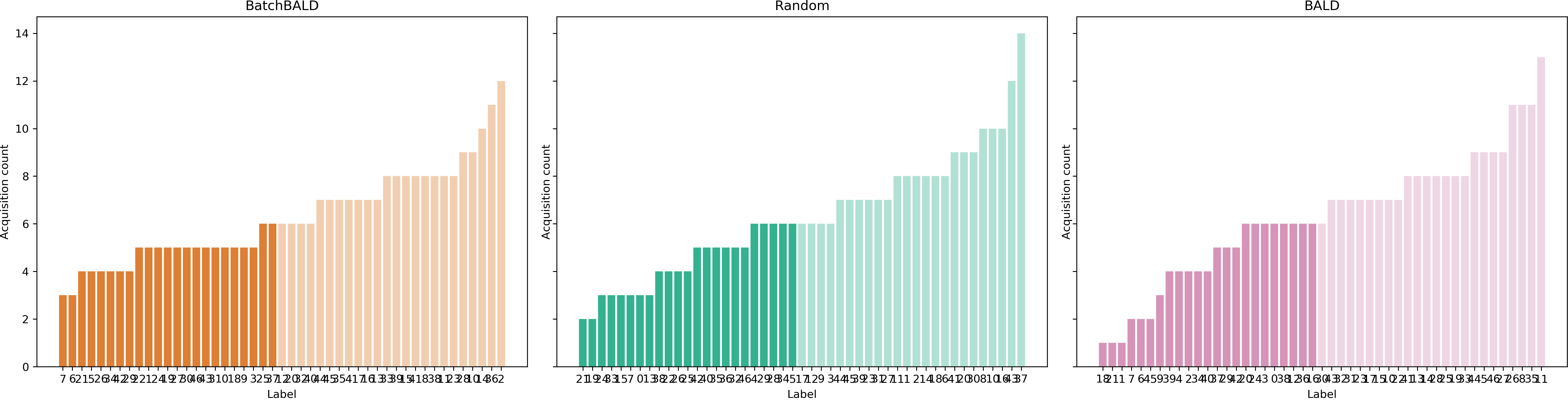}
	\caption{\emph{Histogram of acquired class labels on \emph{EMNIST}.} BatchBALD left and
	BALD right. Classes are sorted by number of acquisitions.
	Several EMNIST classes are underrepresented in BALD and random acquisition while BatchBALD acquires classes
	more uniformly. The histograms were created from all acquired points at the
	end of an active learning loop}
    \label{histogram_labels_emnist}
\end{figure}

\end{document}